%% file: Arxiv_version.tex
\documentclass[10pt]{article}
\usepackage{booktabs,tabularx,array}
\newcommand{\Ccov}{\normalfont{\textsf{C}}_{\textsf{cov}}}
\newcolumntype{Y}{>{\raggedright\arraybackslash}X}

\input{_macros}
\definecolor{OliveGreen}{rgb}{0,0.6,0}
\usepackage{natbib}
\usepackage{custom}
\usepackage[normalem]{ulem}
\usepackage[ruled,vlined]{algorithm2e} 
\allowdisplaybreaks{}
\makeatletter
\renewcommand{\paragraph}{%
  \@startsection{paragraph}{4}%
  {\z@}{1.25ex \@plus 1ex \@minus .2ex}{-1em}%
  {\normalfont\normalsize\bfseries}%
}
\makeatother

\usepackage{bm,bbm}

\newcommand{\sel}{\normalfont{\textsf{Sel}}}

\newcommand{\map}{\normalfont{\textsf{map}}}
\newcommand{\reg}{\normalfont{\textsf{Regret}}}
\newcommand{\vc}{\normalfont{\textsf{VC}}}

\renewcommand{\S}{\normalfont{\textsf{S}}}
\newcommand{\V}{\normalfont{\textsf{V}}}

\newcommand{\D}{\normalfont{\textsf{D}}}
\newcommand{\R}{\normalfont{\textsf{R}}}

\newcommand{\C}{\normalfont{\textsf{C}}}

\newcommand{\rev}[1]{\textcolor{black}{{#1}}}

\renewcommand{\C}{\normalfont{\textsf{C}}}

\renewcommand{\exp}{\normalfont{\textsf{exp}}}

\title{Agnostic Smoothed Online Regression with Adversarial Responses}

\begin{document}
\renewcommand{\thefootnote}{\fnsymbol{footnote}}
\setcounter{footnote}{2}

\author{%
  Xuanyu Chen\textcolor{magenta}{\thanks{Equal contribution. Department of Statistics, University of Michigan; Email: \textcolor{blue}{\texttt{\{xuanyuch,yueuy\}@umich.edu}}}}
  \and
  Yue Yu\textcolor{magenta}{\footnotemark[3]}
}
\maketitle

\begin{abstract}
We study smoothed online prediction with bounded adversarial responses. This widely studied framework bridges i.i.d. sampling and adversarial covariate selection through a smoothness parameter $\textsf{C}_{\textsf{cov}}$, which bounds conditional covariate densities relative to a fixed, unknown base measure. We propose \textsc{Hedge-Cover}, an information-theoretic algorithm that achieves sublinear regret $\widetilde{O}(\sqrt{\text{Pdim}(\mathcal{F}) \textsf{C}_{\textsf{cov}} T})$ for function classes with bounded pseudo-dimension. The algorithm aggregates a carefully constructed family of experts using \textsc{Hedge}, with a prior that links regret to the number of disagreements between a consistent selector and a target function. We bound this number by exploiting covariate smoothness. This answers an open problem posed in \cite{blanchard2025agnostic} on the minimax optimal adaptive regret of the smoothed online regression problem. We establish a matching lower bound for the class of linear predictors. The main intricacy of the lower bound lies in explicitly constructing a challenging sequential covariate distribution supported on mutually orthogonal hyperplanes. This construction may be of independent technical interest. 
Finally, 
we revisit the well-specified setting and quantify the effect of response noise. For conditionally $\nu^2$-subGaussian responses, we extend the existing lower bound under realizable responses by showing that the minimax expected regret is $\Omega((1\vee \nu)\sqrt{(\textsf{C}_{\textsf{cov}}-1)dT})$ for a function class of VC dimension $d$. A corresponding upper bound for ERM matches this dependence on $\nu$.

\end{abstract}

\input{section/main_body}

\bibliographystyle{plainnat}
\bibliography{YY,extra}

\clearpage
\appendix
\section*{Appendix}
\addcontentsline{toc}{section}{Appendix}
\setcounter{tocdepth}{3}
\begingroup
\tableofcontents
\endgroup


\input{section/appendix}

\end{document}


%% file: _macros.tex
\usepackage{comment,url,graphicx,subcaption,relsize,etoc}
\usepackage{amssymb,amsfonts,amsmath,amsthm,amscd,dsfont,mathrsfs,mathtools,multirow,microtype,nicefrac,pifont,bm}
\usepackage{float,psfrag,epsfig,color,url,hyperref}
\usepackage{upgreek}
\usepackage[dvipsnames]{xcolor}
\usepackage{epstopdf,bbm,mathtools,enumitem}
\usepackage[toc,page]{appendix}
\usepackage[mathscr]{euscript}

\usepackage[top=1in, bottom=1in, left=1in, right=1in]{geometry}

\def\balign#1\ealign{\begin{align}#1\end{align}}
\def\baligns#1\ealigns{\begin{align*}#1\end{align*}}
\def\balignat#1\ealign{\begin{alignat}#1\end{alignat}}
\def\balignats#1\ealigns{\begin{alignat*}#1\end{alignat*}}
\def\bitemize#1\eitemize{\begin{itemize}#1\end{itemize}}
\def\benumerate#1\eenumerate{\begin{enumerate}#1\end{enumerate}}

\newenvironment{talign*}
 {\let\displaystyle\textstyle\csname align*\endcsname}
 {\endalign}
\newenvironment{talign}
 {\let\displaystyle\textstyle\csname align\endcsname}
 {\endalign}

\def\balignst#1\ealignst{\begin{talign*}#1\end{talign*}}
\def\balignt#1\ealignt{\begin{talign}#1\end{talign}}

\let\originalleft\left
\let\originalright\right
\renewcommand{\left}{\mathopen{}\mathclose\bgroup\originalleft}
\renewcommand{\right}{\aftergroup\egroup\originalright}

\def\tinycitep*#1{{\tiny\citep*{#1}}}
\def\tinycitealt*#1{{\tiny\citealt*{#1}}}
\def\tinycite*#1{{\tiny\cite*{#1}}}
\def\smallcitep*#1{{\scriptsize\citep*{#1}}}
\def\smallcitealt*#1{{\scriptsize\citealt*{#1}}}
\def\smallcite*#1{{\scriptsize\cite*{#1}}}

\def\<{\left\langle} 
\def\>{\right\rangle}

\DeclareSymbolFont{rsfs}{U}{rsfs}{m}{n}
\DeclareSymbolFontAlphabet{\mathscrsfs}{rsfs}
\ifdefined\nonewproofenvironments\else
\ifdefined\ispres\else
\newtheorem{theorem}{Theorem}[section]

\newtheorem{lemma}[theorem]{Lemma}
\newtheorem{corollary}[theorem]{Corollary}

\theoremstyle{definition}
\newtheorem{definition}{Definition}
\newtheorem{question}{Question}

\theoremstyle{definition}

\renewenvironment{proof}{\noindent\textbf{Proof.}\hspace*{.3em}}{\qed \vspace{.1in}}

\newenvironment{proof-sketch}[1][{}]{\noindent\textbf{Proof Sketch of {#1}}
  \hspace*{1em}}{\qed\\}

\newenvironment{proof-idea}{\noindent\textbf{Proof Idea}
  \hspace*{1em}}{\qed\bigskip\\}
\newenvironment{proof-of-lemma}[1][{}]{\noindent\textbf{Proof of Lemma {#1}}
  \hspace*{1em}}{\qed\\}
\newenvironment{proof-of-corollary}[1][{}]{\noindent\textbf{Proof of Corollary {#1}}
  \hspace*{1em}}{\qed\\}  
\newenvironment{proof-of-claim}[1][{}]{\noindent\textbf{Proof of Claim {#1}}
  \hspace*{1em}}{\qed\\}
  \newenvironment{proof-of-proposition}[1][{}]{\noindent\textbf{Proof of Proposition {#1}}
  \hspace*{1em}}{\qed\\}
\newenvironment{proof-of-theorem}[1][{}]{\noindent\textbf{Proof of Theorem {#1}}
  \hspace*{1em}}{\qed\\}
\newenvironment{proof-attempt}{\noindent\textbf{Proof Attempt}
  \hspace*{1em}}{\qed\bigskip\\}

\theoremstyle{definition}

\theoremstyle{definition}
\newtheorem{remark}{Remark}
\theoremstyle{definition}

\theoremstyle{definition}
\newtheorem{proposition}[theorem]{Proposition}
\newtheorem{claim}[theorem]{Claim}

\theoremstyle{definition}

\fi
\makeatletter
\@addtoreset{equation}{section}
\makeatother

\hypersetup{
  colorlinks,
  linkcolor={blue},
  citecolor={blue},
  urlcolor={magenta}
}



%% file: section/main_body.tex
\section{Introduction}

Online learning is a framework for sequential decision-making in which a learner iteratively
makes predictions, observes feedback, and updates its strategy. Classical PAC learning~\citep{vapnik1971uniform, vapnik1974theory} assumes i.i.d. samples, whereas online learning allows sequential,
potentially adversarial covariates and responses. For binary classification under 0-1 loss, a class is online learnable if and only if it
has finite Littlestone dimension~\citep{littlestone1988learning,shalev2014understanding,alon2021adversarial}. This requirement is substantially stronger than finiteness of the VC dimension, which
characterizes distribution-free PAC learnability for binary classification~\citep{vapnik1971uniform,vapnik1974theory,devroye1996probabilistic,mathiasen2026optimal}.

To accommodate dependence within the covariate sequence and interpolate between the
i.i.d. and fully adversarial regimes, \citet{rakhlin2011online} first introduced the notion of a smoothed adversary in online learning. Subsequent work has developed and analyzed algorithms under several variants of smoothed
online learning; see~\cite{block2022efficient,pmlr-v195-block23b,haghtalab2022oracle,chen2026self, haghtalab2020smoothed,haghtalab2024smoothed,raman2024smoothed,wu2024oracle,blanchard2025agnostic}. In the recent literature, most work considers the case in which the base measure $\mu$ is known (or equivalently, uniform over its support) and the learner is able to sample from $\mu$. See Section~\ref{subsec:related-works} for a detailed comparison with these results.

In the present work, we study the online prediction problem in which no assumptions except boundedness are made on the response values $\{y_t\}_{1\leq t \leq T}$ selected
by the adversary, and the
base measure $\mu$ is only existential and is \textit{not} announced to the learner a priori. We refer to the setting as
the \textit{agnostic} smoothed online learning setting. Arguably, removing knowledge of the base measure weakens the learner’s information and generally makes the problem more realistic but also more challenging. In online regression with squared loss, at each round $t=1,2,\cdots,T$, the forecaster observes $x_t\in \mathbb{R}^d$, outputs a prediction $\widehat{y}_t\in \mathbb{R}$ and receives $y_t\in [-1,1]$ from nature. 
The goal of the learner is to achieve low squared loss compared
to a fixed function class baseline $\mathcal{F}$. Formally, the regret up to $T$ is defined as

\begin{align}
\label{eq:reg}
    \reg(T,\mathcal{F}): = \sum_{1\leq t\leq T}(\widehat{y}_t - y_t)^2 - \inf_{f\in \mathcal{F}}\sum_{1\leq t \leq T}(f(x_t)- y_t)^2.
\end{align}
We highlight that \cite{blanchard2025agnostic} studies a similar agnostic smoothed online prediction problem with arbitrarily adversarial responses and analyzes the\textit{ oblivious regret}
\begin{align}
\label{eq:oblivious}
    \reg^{\textsf{obl}}(T,\mathcal{F}) = \sum_{1\leq t\leq T}\bbE (\widehat{y}_t - y_t)^2 - \inf_{f\in \mathcal{F}}\sum_{1\leq t \leq T}\bbE(f(x_t)- y_t)^2\leq \bbE \,\reg(T,\mathcal{F}).
\end{align}
Unless otherwise specified, we consider the regret defined in~\eqref{eq:reg}, which is sometimes called adaptive regret.

We next formalize a smoothness condition that caps how concentrated each $X_t$ can
be relative to a fixed base measure $\mu$.
An environment $\textsf{Env}$ (i.e. the joint law of $(x_1,y_1,\cdots,x_T,y_T)$) is $\C_{\textsf{cov}}$-smooth if there is a probability measure $\mu$ on
$\mathbb{R}^d$ such that the conditional law of $x_t$, given
$(x_1,\ldots,x_{t-1})$, has a Radon--Nikodym derivative with
respect to $\mu$ at most $\C_{\textsf{cov}}$.
\begin{definition}
\label{def:smoothness}
    There exists a probability measure $\mu$ on $\mathcal{X}\subseteq\mathbb{R}^d$ and a parameter
$\C_{\normalfont{\textsf{cov}}} \geq 1$ such that for every round $t$, given the partial history
$\mathcal{H}_{t-1}^{x}=\sigma(\{x_1,\dots,x_{t-1}\})$, the conditional law of $x_t$, denoted by 
$P_t(\,\cdot\mid\mathcal{H}_{t-1}^{x})$, satisfies
\begin{align*}
\frac{\mathrm{d}P_t(\,\cdot\mid\mathcal{H}_{t-1}^{x})}{\mathrm{d}\mu}(x)
\leq \C_{\textsf{cov}}
\qquad \textit{for $\mu$-a.e. $x\in\mathbb{R}^d$.}
\end{align*}
\end{definition}

Some works in the literature adopt the reciprocal parametrization $\sigma = \C_{\textsf{cov}}^{-1}$ of the smoothness condition, while we stick to $\Ccov$ in the present work. To quantify the benefit of covariate smoothness in online learning, we study the minimax optimal expected adaptive-benchmark regret. For a function class $\mathcal{F}$, we consider the following minimax regret
\begin{align*}
    \mathfrak{R}(T,\mathcal{F})=\mathfrak{R}(T,d,\Ccov,\mathcal{F}) = \inf_{\pi} \sup_{\substack{\textsf{Env}\text{ is}\\ \Ccov-\text{smooth}}}
    \mathbb{E}_{\pi,\textsf{Env}}\big(\reg(T,\mathcal{F})\big),
\end{align*}
where $\pi$ ranges over all (possibly randomized) learners. We advance our understanding of the quantitative behavior of agnostic smoothed online
regression by addressing the following question.

\begin{question}
In the adversarial setting, can one achieve sublinear (adaptive) regret for
smoothed online learning without prior knowledge of the base measure? If so, what is the minimax-optimal regret rate?
\end{question}


We also revisit the well-specified setting, a special case of agnostic smoothed online learning in which the responses satisfy $\bbE[y_t\,|\, x_t] = f^\star(x_t)$ for some $f^\star\in\mathcal{F}$. 
In this setting, empirical risk minimization (ERM) is a natural procedure that does not require knowledge of the base measure. 
\citet{block2024performance} establish an upper bound on the expected cumulative estimation error of ERM in terms of the Wills functional~\citep{wills1973gitterpunktanzahl,hadwiger1975will,vitale1996wills,mourtada:hal-05304667},
while \citet{blanchard2025agnostic} proves a matching worst-case lower bound for VC classes under realizable responses. 
Taken together, these results show that, when $\Ccov$ is bounded away from $1$, ERM is minimax-optimal in its dependence on $\Ccov$, $d$, and $T$, up to logarithmic factors. However, they do not provide a sharp characterization of the dependence on the response-noise level.

We now summarize our main contributions in this paper. 
\begin{enumerate}
\item In the adversarial response setting, we propose \textsc{Hedge-Cover}, an algorithm that requires no knowledge of the base measure and achieves sublinear regret \rev{$\widetilde{O}([\text{Pdim}(\mathcal{F}) + \sqrt{(\Ccov-1)\text{Pdim}(\mathcal{F})T}]\wedge T)$}, as established in Theorem~\ref{thm:upper-bound}. \citet{blanchard2025agnostic} proposes \textsc{R-Cover}, which achieves an oblivious regret bound of $O(\sqrt{\Ccov dT})$ for regression (see~\eqref{eq:oblivious}), but leave open whether sublinear adaptive regret is attainable in this setting. Our result resolves this open question by providing the first sublinear adaptive regret guarantee for regression with adversarial responses and an unknown base measure.

\item We complement the upper bound for \textsc{Hedge-Cover} by constructing a challenging smoothed environment that yields the lower bound $\mathfrak{R}(T)=\Omega(T\wedge \sqrt{d(\Ccov-1)T})$ for $d\geq 2$ and $\Ccov>1$. \rev{This matches the upper bound in Theorem~\ref{thm:upper-bound} up to logarithmic factors when $(\Ccov-1)T\geq d$, including its dependence on $\Ccov-1$.} The proof develops a systematic method for constructing hard environments, which may be of independent theoretical interest.

\item  In the well-specified setting, we quantify the effect of response noise. For conditionally $\nu^2$-subGaussian responses, we prove that there exists a class of VC dimension $d$ for which every learner incurs expected regret $\Omega\big(T\wedge\sqrt{(\Ccov-1)(1\vee \nu^2)dT}\big)$. The bound recovers the known rate in \citet{blanchard2025agnostic} under realizable ($\nu=0$) responses and reveals a missing noise-dependent term in the ERM analysis of \citet{block2024performance}. A minor modification of their argument yields an upper bound with matching dependence on $\nu$. Finally, we show that the Wills-functional upper bound needs not be class-wise tight.


\end{enumerate}

\subsection{Related work}
\label{subsec:related-works}
\paragraph{Online prediction with a smooth environment.}

We review the closest results in smoothed online learning in Table~\ref{tab:regression-classification} and list the work on online regression with a known base measure as a comparison. 
\begin{table}[htbp]
\centering
\caption{
Upper and lower regret bounds for regression under
$\Ccov$-smooth covariates.
}
\label{tab:regression-classification}
\normalsize
\setlength{\tabcolsep}{4pt}
\renewcommand{\arraystretch}{1.05}
\renewcommand{\tabularxcolumn}[1]{m{#1}}

\begin{tabularx}{\textwidth}{
    @{}>{\centering\arraybackslash}m{0.10\textwidth}
    >{\centering\arraybackslash}X
    >{\centering\arraybackslash}X@{}
}
\toprule
& \textbf{Adversarial responses}
& \textbf{Well-specified responses} \\
\midrule

\shortstack{Upper\\bound}
&
\textbf{Unknown $\mu$, oblivious.}

\textsc{R-Cover}:
$O\big(\log^3 T\sqrt{\Ccov dT}\big)$,

for $\textsf{fat}_{\mathcal F}(r)\le d\log(1/r)$.

\cite[Theorem 4]{blanchard2025agnostic}

\medskip

\textbf{Unknown $\mu$, adaptive.}

\textsc{Hedge-Cover}:

\rev{$\widetilde{O}(\sqrt{(\Ccov-1)\text{Pdim}(\mathcal{F})T})$}.

Theorem~\ref{thm:upper-bound}

\medskip

\textbf{Known $\mu$, adaptive.}

$O\big(\log^{3/2}T\log(\Ccov T)\sqrt{dT}\big)$.

\cite[Theorem 3]{pmlr-v178-block22a}
&
\textbf{Unknown $\mu$, estimation error.}

\textsc{ERM}:
$\widetilde O\big(\sqrt{\Ccov dT}\big)$.

\cite[Theorem 8]{block2024performance}
\\
\midrule

\shortstack{Lower\\bound}
&
\textbf{Unknown $\mu$, adaptive.}

$\Omega\Big(\sqrt{(\Ccov-1)dT}\wedge T\Big)$

for all algorithms.

Theorem~\ref{theorem:minimax}
&
\textbf{Unknown $\mu$, adaptive.}

$\Omega\Big(\sqrt{(\Ccov-1)dT}\wedge T\Big)$

for all algorithms.

Realizable responses:

\cite[Theorem 2]{blanchard2025agnostic}

Noise-dependent extension:

Theorem~\ref{thm:lower-bound-well-specified}
\\
\bottomrule
\end{tabularx}
\end{table}
In the online classification setting, \citet{blanchard2025agnostic} introduces \textsc{R-Cover}, a recursive-covering algorithm that provides the first sublinear regret guarantee for agnostic smoothed online learning without prior knowledge of $\mu$. \rev{For regression classes with pseudo-dimension $d$, the high-probability bound for \textsc{R-Cover} in \citet[Theorem~12]{blanchard2025agnostic} can be extended with slight tinkering, to $\widetilde O(\sqrt{d\Ccov T})$ adaptive regret under smoothness conditional on the \textit{full history} including both the past covariates and responses 
(see Appendix~\ref{sec:app_ERM} for the formal definition)
. Our Theorem~\ref{thm:upper-bound} requires smoothness only conditional on past covariates, which is a weaker smoothness condition that allow past responses to reveal information about future covariates.}
\citet{blanchard2025agnostic} also establishes a minimax lower bound under realizable responses for a function class of VC dimension $d$. We establish a lower bound of the same order for the linear function class when $d\geq 2$. For this class, we further characterize a sharp phase transition: the minimax regret scales as $\sqrt{T}$ when $d\geq 2$ and $\Ccov>1$, but is logarithmic in $T$ when $d=1$ or $\Ccov=1$.


\paragraph{Online linear regression.}
The central ingredient in our lower bound for agnostic regression (Theorem~\ref{theorem:minimax}) over a function class satisfying $\vc(\mathcal{F})=d$ is an explicit construction of an unknown discrete base measure $\mu$. The construction applies even to linear predictors. Online linear regression has a long history~\citep{vovk1997competitive,azoury2001relative,cesa2006prediction}. More recently, by exploiting a connection with self-normalized martingales, \citet{chen2026self} establish a regret upper bound of $O(\sqrt{d\Ccov T\log T})$. As a minor technical contribution, we sharpen this bound to $O(\sqrt{d\Ccov T})$ by using an argument that avoids direct decoupling, which matches the worst regret even for the logarithmic term.

\paragraph{Hybrid online learning.}
Hybrid online learning~\citep{lazaric2009hybrid} studies the setting in which covariates are sampled i.i.d. from an unknown distribution, while the responses may be chosen adversarially. Note that this corresponds to the special case of smoothed online learning with $\C_{\textsf{cov}} = 1$. For an unrestricted response adversary, the best-known results leave a gap between statistical and computational guarantees. Even in this special case, it remains open whether statistical optimality and computational efficiency can be achieved simultaneously, suggesting a possible statistical--computational trade-off in smoothed online learning. Specifically, statistically optimal algorithms are typically computationally intractable for natural hypothesis classes \citep{lazaric2009hybrid, wu2022expected}, whereas known ERM-oracle-efficient algorithms attain statistically suboptimal regret rates \citep{wu2024oracle}. 
Under the additional structural assumption that, at each round, the adversary selects a possibly different labeling function from a fixed and known class, \citet{okoroafor2026oracle} develop an oracle-efficient learning algorithm whose regret is near-optimal up to logarithmic factors and the dependence on the adversary’s labeling class.

\rev{
\textbf{Concurrent work.} After submission, we became aware of two independent concurrent works \citep{voitovych2026oracle,buzaglo2026oracle}. For binary classification with i.i.d. covariates ($\normalfont{\textsf{C}}_{\textsf{cov}}=1$) and adversarial labels, \citet{buzaglo2026oracle} show that their oracle-efficient \textsc{Gaussian-FTPL} algorithm achieves regret of order $(T\cdot \textsf{VC}(\mathcal{F}))^{1/2}(\log T)^{5/2}$. This raises the question of whether randomized algorithm can yield an oracle-efficient algorithm in our setting, with computational complexity measured in empirical risk minimization (ERM) oracle calls.
\citet{voitovych2026oracle} study agnostic smoothed online learning with an unknown base measure under full-history smoothness. Their algorithm uses one weighted ERM oracle call per round for binary classification, achieving expected regret $\widetilde{O}(d\sqrt{\Ccov T})$ for classes of VC dimension $d$. An averaged variant handles real-valued classes and convex Lipschitz losses. While achieving oracle efficiency, their guarantee remains suboptimal in $d$ by a factor of $\sqrt d$.
}

\subsection{Notations}
Let $\mathbb N_{+}=\{1,2,\ldots\}$ and $[n]=\{1,\ldots,n\}$ for $n\in\mathbb N_{+}$. For a finite set $A$, let $|A|$ denote its cardinality, and let $\mathbbm{1}_{\textsf{E}}$ denote the indicator of an event $\textsf{E}$. For sets $A,B\subseteq\mathcal X$, their symmetric difference is denoted by $A\Delta B$. Uppercase letters denote random variables and the corresponding lowercase letters denote their realizations. Let $\mathcal H_t^{x}:=\sigma(\{x_1,\ldots,x_t\})$ denote the covariate history and let $\mathcal H_t:=\sigma(\{x_s,\widehat y_s,y_s\}_{1\leq s\leq t})$ denote the full interaction history. We abbreviate $P_t(\,\cdot\,):=\mathcal L(x_t\mid\mathcal H_{t-1}^{x})$ and $P_t(A):=\mathbb P(x_t\in A\mid\mathcal H_{t-1}^{x})$. For probability measures $P$ and $\mu$, $P\ll\mu$ denotes absolute continuity, $\d P/\d\mu$ denotes the Radon--Nikodym derivative, and $\mu^{\otimes k}$ denotes the $k$-fold product measure. Finally, we use standard asymptotic notation: $a\lesssim b$ and $a\gtrsim b$ mean that the corresponding inequality holds up to a universal positive constant, while $a\asymp b$ means that both relations hold. The notation $\widetilde O(\cdot)$ suppresses polylogarithmic factors in $T$.

\section{\textsc{Hedge-Cover} algorithm}

Before formally introducing the algorithm~\textsc{Hedge-Cover}, we introduce the necessary shorthand notation and define a selection procedure similar to the realizable
device in~\cite{blanchard2025agnostic} as follows. The learner first discretizes the range $[-1,1]$ with a grid $\mathcal{N}_m = \{-1 + 2m^{-1} j\}$ where $j\in [m]$.
With a slight abuse of notation, we use $\mathcal{N}_m\circ \mathcal{F} = \{\textsf{proj}(f(x),\mathcal{N}_m):f\in \mathcal{F}\}$.
For a trajectory  $\D = \{(x_i,\widetilde{y}_i)\}_{1\leq i \leq n}$, where $x_i\in \mathcal{X}$ and $\widetilde{y}_i\in \mathcal{N}_m$, a selector $\sel:(\mathcal{X}\times \mathcal{N}_m)^{\infty}\to \mathcal{N}_m\circ\mathcal{F}$ maps to a (non-zero) function in $\mathcal{F}$ if $\D$ is realizable.
\begin{definition}[Selector]
Fix a well-ordering of $\mathcal{N}_m\circ\mathcal{F}$. Given a trajectory
$\D = \{(x_i,\widetilde{y}_i)\}_{i=1}^{n}$, where
$x_i\in\mathcal{X}$ and $\widetilde{y}_i\in\mathcal{N}_m$,
define $\sel(\D)$ to be the first function
$f_\star\in\mathcal{N}_m\circ\mathcal{F}$ in this ordering satisfying
$f_\star(x_i)=\widetilde{y}_i$ for every $i\in[n]$.
If no such function exists, set $\sel(\D)\equiv 0$, the zero function.
\end{definition}

To apply the \textsc{Hedge} algorithm~\citep{cesa2006prediction}, we introduce hypotheses of the form $\bm h=(\mathcal{T}_{\bm h},\map_{\bm h})$, where $\mathcal{T}_{\bm h}\subseteq[T]$ and $\map_{\bm h}:\mathcal{T}_{\bm h}\to\mathcal{N}_m$. Each hypothesis specifies a set of times $\mathcal{T}_{\bm h}$ at which predictions are prescribed and a mapping $\map_{\bm h}$ that assigns a value in $\mathcal{N}_m$ to each such time. At the remaining times, predictions are determined by the selector. Both the selector $\textsf{Sel}$ and the hypotheses implicitly depend on the discretization $\mathcal{N}_m$. Let $\mathcal{H}$ denote the set of all such hypotheses. Since $|\mathcal{N}_m|=m$, a direct calculation gives $|\mathcal{H}|=\sum_{k=0}^{T}\binom{T}{k}m^k=(m+1)^T$. Before describing \textsc{Hedge-Cover}, we emphasize that it is an information-theoretic algorithm; its computational aspects are beyond the scope of this work.

\begin{algorithm}[htbp]
\LinesNumbered
\caption{\textsc{Hedge-Cover}}
\SetNlSty{textbf}{}{} 
\label{alg:meta}
\KwIn{Time horizon $T$, function class $\mathcal{F}$}
Set $m\gets 2T$;

\tcp{\textsc{Hedge} initialization}

\For{$\bm h \in \mathcal{H}$}
{$\displaystyle p_{0,\bm h}\propto \frac{1}{(|\mathcal{T}_{\bm h}| +1)^2 (m+1)^{|\mathcal{T}_{\bm h}|} \binom{T}{|\mathcal{T}_{\bm h}|}},\quad \sum_{\bm h\in \mathcal{H}}p_{0,\bm h}=1$

$L_{0,\bm h}\gets 0$, $\D_{0,\bm h}\gets \emptyset$;
}

\For{$t\in [T]$}{
Observe $x_t$;

\For{$\bm h \in \mathcal{H}$}{

\label{line:output}
$\widehat{y}_{t,\bm h} = 
\begin{cases}
    \map_{\bm h}(t)\quad&\text{if }t\in \mathcal{T}_{\bm h}\\
    \Big[\sel(\D_{t-1,\bm h})\Big](x_t)\quad& \text{if }t\notin \mathcal{T}_{\bm h}
\end{cases}$
}

$\widehat{y}_t\gets \textsc{Hedge-prediction}(\{\widehat{y}_{t,\bm h}\})$;

Observe $y_t$;

$\{L_{t,\bm h}\}_{h\in \mathcal{H}}\gets \textsc{Hedge-update}(y_t)$;

\For{$\bm h \in \mathcal{H}$}{
$\D_{t,\bm h}\gets \D_{t-1,\bm h}\bigcup \{(\bm x_t,\widehat{y}_{t,\bm h})\};$}
}
\Return $\{\widehat{y}_t\}_{1\leq t \leq T}$
\end{algorithm}

\textsc{Hedge-Cover} proceeds in the following steps for each round.
\begin{itemize}
    \item We initialize a prior weight $p_{0,\bm h}$ for each hypothesis $\bm h\in\mathcal H$ and set its cumulative loss $L_{0,\bm h}$ to zero. Each hypothesis also maintains its own trajectory, initialized as $\D_{0,\bm h}=\emptyset$. The prior is chosen so that hypotheses with fewer prescribed prediction times incur a smaller complexity penalty in the regret bound.

    \item At round $t$, after observing $x_t$, each hypothesis $\bm h$ predicts $\map_{\bm h}(t)$ if $t\in\mathcal T_{\bm h}$; otherwise, it applies the selector $\sel$ to its trajectory $\D_{t-1,\bm h}$ and evaluates the selected function at $x_t$. Thus, each hypothesis follows a consistent selector, with prescribed predictions at the times in $\mathcal T_{\bm h}$.

    \item The learner aggregates these predictions using \textsc{Hedge-prediction}, assigning weights proportional to $p_{0,\bm h}\exp(-\eta L_{t-1,\bm h})$. After observing $y_t$, it adds $(\widehat y_{t,\bm h}-y_t)^2$ to each hypothesis's cumulative loss and appends $(x_t,\widehat y_{t,\bm h})$ to that hypothesis's trajectory.
\end{itemize}
The role and the reason for our choice of the initialization $\{p_{0,\bm h}\}_{\bm h \in \mathcal{H}}$ will be explained shortly after the proof sketch in Remark~\ref{remark:init}.
\begin{definition}[pseudo dimension]
For a function class $\mathcal{F}$ consisting of functions $\mathcal{X} \to [-1,1]$, its pseudo dimension is defined as 
\begin{align*}
    \normalfont{\text{PDim}}(\mathcal{F}) = \vc(\{(x,r):r\leq f(x),f\in \mathcal{F}\}).
\end{align*}
\end{definition}
Pseudo dimension extends VC dimension to real-valued function classes by allowing a distinct real threshold at each shattered point~\citep{pollard2012convergence}. We provide an upper bound on the expected regret as follows. The formal high probability version of the theorem is stated in Theorem~\ref{thm:upper-bound-formal} in Appendix~\ref{appendix:proof-upper-bound}.

\begin{theorem}
\label{thm:upper-bound}
There exists an algorithm such that, for each function class $\mathcal{F}$ with $\normalfont{\text{Pdim}}(\mathcal{F}) \leq d$ and each $\Ccov$-smooth environment $\normalfont{\textsf{Env}}$,
\begin{align}
\label{eq:upper-expectation}
\bbE\big[\reg(T,\mathcal{F})\big]= \widetilde{O}\big(\sqrt{\Ccov d T}\wedge T\big).
\end{align}
\rev{If $\Ccov>1$, we further have \begin{align}
\label{eq:upper-expectation-cov}
\bbE\big[\reg(T,\mathcal{F})\big]=\widetilde{O}\big(\big[d+\sqrt{(\Ccov-1)dT}\big]\wedge T\big).
\end{align}}
\end{theorem}
\begin{remark}
\label{remark:pseudo}
\rev{For $d\geq 2$ and $\Ccov>1$, the additive $d$ term is absorbed for sufficiently large $T$, and the dependence on $\Ccov-1$ matches the lower bound in Theorem~\ref{theorem:minimax} up to logarithmic factors.}
We measure the complexity of $\mathcal{F}$ by its pseudo dimension, as in \cite{haghtalab2022oracle}. This assumption is stronger than the scale-dependent fat-shattering conditions considered by \cite{pmlr-v178-block22a} and \cite{blanchard2025agnostic}: finite pseudo dimension bounds the fat-shattering dimension uniformly over all positive scales, whereas the latter conditions permit it to diverge as the scale vanishes. Our analysis uses this uniform control to bound the growth of discretized hypotheses and their disagreement sets. Accordingly, our adaptive regret guarantee addresses the unknown base-measure problem for classes of finite pseudo dimension. Extending it to the broader classes allowed by fat-shattering growth conditions remains open.
\end{remark}
\begin{proof-sketch}[Theorem~\ref{thm:upper-bound}]

\underline{\textbf{Step 1: Covering.}}
For sufficiently large $T$, we fix a finite cover of the discretized function class $\mathcal{M}_1\subseteq \mathcal{N}_m\circ \mathcal{F}$ such that, for each $f\in \mathcal{N}_m\circ \mathcal{F}$, there exists $g\in \mathcal{M}_1$ such that
\begin{align}
\label{eq:cover}
    \mu\{x:f(x)\neq g(x)\} \leq \varepsilon
\end{align}
with $\log|\mathcal{M}_1|= O(d\log(T\varepsilon^{-1}))$. This cover is used only in the analysis.

\underline{\textbf{Step 2: From a function to a hypothesis $\bm h$.}}
Given the preceding class $\mathcal{M}_1$, for any $f_j\in \mathcal{M}_1$, we consider the following set:
\begin{align*}
    \textsf{B}(j,t):=\big\{x\in \mathcal{X}:\sel(\{x_i,f_j(x_i)\}_{1\leq j\leq t-1})(x)\neq f_j(x)\big\}.
\end{align*}
As $\mathcal{M}_1\subseteq \mathcal{N}_m \circ \mathcal{F}$, the selector $\sel$ cannot be trivially set to $0$, as $f_j$ certainly satisfies the definition. Though there is apparently no guarantee that the selector can be the same as $f_j$, the observed covariates $\{x_1,\cdots,x_{t-1}\}$ cannot belong to $\textsf{B}(j,t)$. From the set, one can construct a good enough hypothesis $\bm h = (\mathcal{T}_{\bm h} := \{t:x_t\in \textsf{B}(j,t)\}, \map_{\bm h}(t) = f_j(x_t))$. Such a construction ensures that $\widehat{y}_{t,\bm h} = f_j(x_t)$ for every $t$.

\underline{\textbf{Step 3: Hedge bound.}}
For a sufficiently small $\eta$, \textsc{Hedge} guarantees that, for any hypothesis $\bm h$, we can relate the cumulative squared loss induced by Algorithm~\ref{alg:meta} to that of the predictions by $\bm h$:
\begin{equation}
\label{eq:sketch}
\begin{aligned}
\sum_{1\leq t \leq T}(\widehat{ y}_t - y_t)^2 -  \sum_{1\leq t \leq T}(f_j(x_t)- y_t)^2
&  =
    \sum_{1\leq t \leq T}(\widehat{ y}_t - y_t)^2 -  \sum_{1\leq t \leq T}(\widehat{ y}_{t,\bm h}- y_t)^2 \\
    & \leq -\frac{1}{\eta} \log(p_{0,\bm h})  \lesssim \eta^{-1} (1 + |\mathcal{T}_{\bm h}|)\log(T),
\end{aligned}
\end{equation}
as $m\asymp T$. See Lemma~\ref{lemma:hedge} for the details.

\underline{\textbf{Step 4: Upper bound on $|\mathcal{T}_{\bm h}|$.}}
From Steps~1--3, it suffices to show that, under the smooth environment, the cardinality of $\mathcal{T}_{\bm h}=\sum_{1\leq t \leq T}\mathbbm{1}\{x_t\in \textsf{B}(j,t)\}$ is of order $\widetilde{O}(\sqrt{d\Ccov T})$. This is arguably the most technically challenging step, where we use the decoupling techniques and martingale arguments. In order to strengthen the result to its high-probability version, we also invoke a Vapnik-Chervonenkis inequality. This step is formalized in Proposition~\ref{prop:good-event}.

\underline{\textbf{Step 5: Putting everything together.}}
To relate the left-hand side of \eqref{eq:sketch} to regret, it remains to consider (1) the discretization error and (2) the bad event in \eqref{eq:cover}. The former is at most $O(1)$ as $m\asymp T$, and the latter can be controlled by carefully choosing $\varepsilon$ in terms of $T$. Precisely, we choose $\varepsilon \asymp (\Ccov T \log T)^{-1}$ to accommodate the extra $T\log T$ terms that occur in the decoupling step.
\end{proof-sketch}
\begin{remark}
[Initialization $p_{0,\bm h}$]
The initialization of the distribution over $\mathcal{H}$ is designed to make the aggregation penalty depend on the number of corrections required by a hypothesis. As shown in Step 3 of the proof sketch, the regret is controlled, modulo approximation and discretization errors, by the negative log initialized weight $\log(1/p_{0,\bm h})$ of a suitable hypothesis $\bm h$. For each $\tau\in\{0,\ldots,T\}$, there are $\binom{T}{\tau}m^\tau$ hypotheses satisfying $|\mathcal T_{\bm h}|=\tau$. The factor $\binom{T}{\tau}(m+1)^\tau$ in the denominator accounts for this multiplicity, while the additional factor $(\tau+1)^2$ keeps the normalizing constant bounded uniformly in $T$. Consequently, when $m\asymp T$, we have $\log(1/p_{0,\bm h})=O(1+\tau\log(T))$. For the hypothesis constructed to track $f_j\in \mathcal{M}_1$, $\tau$ equals the number of rounds on which the selector disagrees with $f_j$. Thus, a sufficiently tight bound on this count $\tau$ yields the desired aggregation regret bound.
\label{remark:init}
\end{remark}

\subsection{Decoupling}
As in much of the smoothed online learning literature, a key step in our analysis relies on the following decoupling technique, first introduced by \citet[Theorem~2.1]{haghtalab2024smoothed}.
\begin{lemma}[Theorem~2.1 \citep{haghtalab2024smoothed}]
\label{lemma:decouple}
Let $\mathcal X$ be a Borel subset of Euclidean space, and let
$P_t=\mathcal L(X_t\mid\mathcal H^x_{t-1})$, $1\leq t\leq T$, be
$\Ccov$-smooth with respect to a fixed probability measure $\mu$.
For each $k\in\mathbb N_+$ there is a coupling $\Pi$ of
$(X_t,Z_1^{(t)},\ldots,Z_k^{(t)})_{1\leq t\leq T}$ such that:
\begin{enumerate}[label=\normalfont{(\textsf{\alph*.})}]
\item The covariates $(X_1,\ldots,X_T)$ have the prescribed sequential law.
\item All the variables $Z_i^{(t)}$, $i\in[k]$, $t\in[T]$, are i.i.d. with law
$\mu$.
\item Conditional on $X_1,\ldots,X_{t-1}$, the variables
$\{Z_i^{(s)}:s\geq t,\ i\in[k]\}$ are i.i.d. with law $\mu$.
\item For each $t\leq T$, with probability at least
$1-t(1-\Ccov^{-1})^k$, simultaneously for all $s\leq t$,
\begin{align*}
 X_s\in\{Z_1^{(s)},\ldots,Z_k^{(s)}\}.
\end{align*}
\end{enumerate}
\end{lemma}

We will also need the following non-asymptotic inequality for classes of functions whose complexity is described through subgraphs. There are multiple ways to prove results stated in Claim~\ref{claim}. Here we detail the approach via the classical Vapnik-Chervonenkis inequality \citep{anthony1993result}. A left-tail bound is proved and the constants have recently been improved in \cite{portier2026uniform}. We only need the inequality in~\eqref{eq:left}.
\begin{lemma}[Theorem 2.1 \citep{anthony1993result}; Theorem~1 \citep{portier2026uniform}]

\label{lemma:VC-ineq}
Let $Z_1,\cdots,Z_n\overset{i.i.d.}{\sim } P$ and $\mathcal{C}$ be a countable class of measurable subsets of $\mathcal{X}$. We have, for all $\zeta>0$,
\begin{align}
\label{eq:left}
    \bbP\Big(\sup_{C\in \mathcal{C}} \frac{P_n(C) - P(C)}{\sqrt{P_n(C)}}>\zeta\Big)\leq \Pi_{\mathcal{C}}(2n)\,\exp\big(-\normalfont{\textsf{c}} n \zeta^2\big).
\end{align}
\end{lemma}

Now we are ready to prove a result holding for a general countable set class.
\begin{proposition}
\label{prop:good-event}
For some base measure $\mu$, $x_1,\cdots,x_T\sim P_t(\cdot)$ satisfy the $\Ccov$-smoothness condition, as described in Definition~\ref{def:smoothness}. Let $\mathcal{A}$ be a countable set class. With probability at least $1-\delta$, the following events hold simultaneously:
\begin{enumerate}
    \item every pair of sets $A_1,A_2\in \mathcal{A}$ such that $\mu(A_1\Delta A_2)\leq \varepsilon:=\frac{\normalfont{\textsf{c}}_{1}}{\Ccov T\log(T\delta^{-1})}$ satisfies
    \begin{align}
    \label{eq:event1}
         \sum_{1\leq t \leq T}\mathbbm{1}\Big\{x_t\in A_1\Delta A_2\Big\}\lesssim \log\Big(\delta^{-1} \Pi_{\mathcal{A}}(2T\Ccov \log(T\delta^{-1}))\Big);
    \end{align}
    \item simultaneously for any $1\leq t'\leq T$ and set $A\in \mathcal{A}$, the intersection $A\cap \{x_1,\cdots,x_{t'}\} = \emptyset$ implies that 
    \begin{align}
        \sum_{1\leq t \leq t'}P_t(A)\lesssim
        \log\Big(\delta^{-1} \Pi_{\mathcal{A}}(2T\Ccov \log(T\delta^{-1}))\Big)
    + \log\Big(\frac{|\mathcal{A}'|}{\delta}\Big),
    \end{align}
    where $\mathcal{A}'$ is a finite $\varepsilon$-proper covering of $\mathcal{A}$.
\end{enumerate}
\end{proposition}
We denote the joint events described in Proposition~\ref{prop:good-event} as $\textsf{E}$.

\rev{
\begin{corollary}
\label{cor:control-strong}
If the event $\normalfont{\textsf{E}}$ occurs, then for fixed predictable events $\{A_t\}\subseteq \mathcal{X}$ (i.e. $A_t$ is $\mathcal{H}_{t-1}^x$ measurable) satisfying $A_t\cap \{x_1,\cdots x_{t-1}\} = \emptyset$,  we have
\begin{equation}
\begin{aligned}
\sum_{1\leq t\leq T}P_t(A_t)
&\lesssim\Big\{\log(2T)\Big[1+\log\Big(\delta^{-1}\Pi_{\mathcal{A}}(2T\Ccov\log(T\delta^{-1}))\Big)+\log\Big(\frac{|\mathcal{A}'|}{\delta}\Big)\Big]
\\
&\quad+\sqrt{(\Ccov-1)T\Big[1+\log\Big(\delta^{-1}\Pi_{\mathcal{A}}(2T\Ccov\log(T\delta^{-1}))\Big)+\log\Big(\frac{|\mathcal{A}'|}{\delta}\Big)\Big]}\Big\}\wedge T.
\end{aligned}
\end{equation}
\end{corollary}
}

\section{Lower bound for online prediction in a smooth environment}
\label{sec:lower-bound}

\begin{theorem}
\label{theorem:minimax}
For sufficiently large $T\in \mathbb{N}^{+}$, there exists an adversarial environment $\normalfont{\textsf{Env}}$ and a function class with finite pseudo dimension $\normalfont{\text{Pdim}}(\mathcal{F})\leq d$ such that for any learner $\pi$, the expected regret satisfies
\begin{align}
   \bbE\,\reg(T,\mathcal{F})\,\textcolor{black}{\gtrsim}\,
    \begin{cases}
        \sqrt{d(\C_{\textsf{cov} }-1)T}\quad&\text{if}\quad \C_{\textsf{cov}}>1\text{ and }d\geq 2;\\
        d\log(1 + T/d)\quad&\text{if}\quad \C_{\textsf{cov}}= 1\text{ or }d = 1.
    \end{cases}
\end{align}
\end{theorem}

\begin{remark}[Rate comparison]
We now give a detailed comparison between our rate and existing work. First, our results characterize a phase transition between $\Ccov=1$ (i.e., when the covariates are drawn i.i.d.) and $\Ccov>1$. It is well known that logarithmic regret is attainable in the classical PAC learning setting. 
\rev{
Second, our results complement the $\Omega(\sqrt{d(\Ccov-1)T})$ lower bound of \citet[Theorem~3]{blanchard2025agnostic} for general function classes of pseudo-dimension $d\geq1$. Our construction specifically considers linear predictors, imposing additional geometric constraints on the hard environments. Somewhat surprisingly, for $d\geq2$, the linear predictor class already captures the worst-case difficulty among function classes with pseudo-dimension $\operatorname{Pdim}(\mathcal F)\leq d$. Moreover, together with the matching upper bounds in \citet[Theorem~5]{chen2026self} (for $d=1$) and Proposition~\ref{prop:upper-unregularized-VAW} (for $d\geq2$), our lower bounds reveal a sharp dimension-dependent transition for linear predictors, from logarithmic dependence on $T$ when $d=1$ to $\sqrt{T}$ dependence when $d\geq2$ and $\Ccov>1$. Since the $\sqrt T$ lower bound of \citet{blanchard2025agnostic} holds for general classes of pseudo-dimension one, this transition reflects the additional structural properties of the one-dimensional linear class rather than dimension alone.}

Finally, our lower bound rules out logarithmic dependence on the smoothness parameter $\Ccov$ in the agnostic setting. This contrasts with the setting in which one knows and can sample from $\mu$, where logarithmic dependence is possible, aligned with the observations by \cite{block2022efficient} and  \cite{wu2023online}.

\end{remark}

As the idea to prove the lower bound will also be used in the proof of the lower bound in Theorem~\ref{thm:lower-bound-well-specified} later, we sketch some proof heuristic as follows. The complete proof is relegated to Appendix~\ref{appendix:lower-bound}.
 
\begin{proof-idea}

\rev{
We establish the lower bound for the unrestricted linear class $\mathcal{F}_{\textsf{lin}}:= \{x\mapsto \theta^\T x:x\in \mathcal{X}\}$. To obtain a bounded function class of pseudo-dimension $d$, we consider the clipped linear class $\mathcal{F}_{\textsf{clip}}:= \{x\mapsto (\theta^\T x\vee -1)\wedge 1:x\in \mathcal{X}\}$. Since clipping can only decrease the comparator loss, the lower bound for $\mathcal F_{\textsf{lin}}$ also holds for $\mathcal F_{\textsf{clip}}$. 
}
To derive the lower bound, we explicitly construct a hard sequential law by constructing $d/2$ trees of depth $\Theta( d^{-1/2} (\Ccov T)^{1/2})$, when $\Ccov>1$ and $d =2$. Each visited level of a two-dimensional tree contributes at
least $1/2$ expected regret for any learner. The corner cases $d=1$ or $\C_{\textsf{cov}}=1$ correspond to the case where each $x_t$ lies on the same line or, at each round, $\{x_t\}_{1\leq t \leq T}$ are independent.  
\end{proof-idea}
\rev{
Although Theorem~\ref{theorem:minimax} is stated for bounded function classes, its proof establishes the lower bound directly for the unrestricted linear class $\mathcal F_{\textsf{lin}}$. For this class, in the practically relevant regime $d\geq 2$ and $\C_{\textsf{cov}}>1$, \citet[Theorem~8]{chen2026self} show that the VAW estimator has regret of order $O(\sqrt{d\C_{\textsf{cov}}T\log T})$. In the following proposition, we sharpen this guarantee by removing the extra $\sqrt{\log T}$ factor. Together with our lower bound, it establishes the optimality of the VAW estimator for the unrestricted linear class when $d\geq2$ and $\Ccov>1$, without any logarithmic gap. 
}
\begin{proposition}
\label{prop:upper-unregularized-VAW}
For every $\delta\in(0,1)$, with probability at least $1-\delta$ over the covariate sequence $\{x_t\}_{1\leq t\leq T}$, the VAW predictor satisfies
\begin{equation}
\sum_{1\leq t \leq T}(\widehat{y}_t - y_t)^2- \inf_{\theta\in \bbR^d}\sum_{1\leq t \leq T}(\theta^\T x_t  -y_t)^2
\lesssim \sqrt{d\C_{\textsf{cov}}T}
+\log(1/\delta),
\end{equation}
simultaneously for every outcome sequence $y_1,\ldots,y_T\in[-1,1]$. Consequently, the VAW predictor is optimal, up to universal constants, for all $d\geq 2$, $\Ccov>1$, and $T$.
\end{proposition}
\rev{
The choice of the comparator class is essential here. Against the clipped linear class $\mathcal F_{\textsf{clip}}$, the VAW estimator can incur linear regret even when $d=1$ and $\Ccov=2$. Neither clipping its predictions nor incorporating clipping into its optimization objective guarantees sublinear regret in general. We provide explicit counterexamples for these variants in Appendix~\ref{sec:app_VAW}.
}
The proof of Proposition~\ref{prop:upper-unregularized-VAW} proceeds largely as in \citet[Theorem 8]{chen2026self} by carefully avoiding the rejection sampling and coupling lemma originally proposed in~\cite{haghtalab2024smoothed}. The rejection sampling step leads to the maximum of $T$ geometric random variables $J_t\overset{i.i.d.}{\sim }\textsf{Geom}(\C_{\textsf{cov}}^{-1})$, thus for each round, one must sample $\Theta(\C_{\textsf{cov}}\log (T))$ times.  Instead, we use the smoothness by relating the law of the subsequence $\bm x_{I}$ to its i.i.d. dominating measure counterpart $\mu_{\otimes |I|}$.

We now turn to the well-specified setting, starting with the following definitions. 
\begin{definition}
An environment is well-specified with respect to a function class $\mathcal{F}$ if there exists an $\mathcal{H}_0$-measurable function $f^{\star} \in \mathcal{F}$ such that, for every $t \in[T]$, $\bbE[y_t \,|\, \mathcal{H}_{t-1}, x_t]=f^{\star}(x_t)$. The environment is realizable if $y_t=f^{\star}\left(x_t\right)$, and it is $\nu^2$-subGaussian if $y_t = f^\star(x_t) + \eta_t$ with $\eta_t\,|\,\mathcal{H}_{t-1},x_t$ being mean-zero and $\nu^2$-subGaussian. 
\end{definition}
ERM is a natural algorithm for well-specified environments. At each round $t$, the learner chooses 
\begin{equation}\label{eq:ERM}
\widehat{f}_t \in \operatorname*{argmin}_{f\in\mathcal{F}}
\sum_{s=1}^{t-1}\big(f(x_s)-y_s\big)^2,
\end{equation}
the minimizer of the empirical error on historical observations, and predicts $\widehat y_t=\widehat f_t(x_t)$. 

Under realizable responses, \citet{blanchard2025agnostic} constructs a function class of VC dimension $d$ for which the minimax expected regret is $\Omega(\sqrt{(\Ccov-1)dT})$.
We refine this result by quantifying the additional difficulty caused by response noise. The following theorem extends their lower bound to conditionally $\nu^2$-subGaussian responses and makes its dependence on the noise level explicit. 


\begin{theorem}
\label{thm:lower-bound-well-specified}

For any $d\in\bbN_+$, there exists a function class $\mathcal{F}:\mathcal{X}\to \{0, 1\}$ with $\vc(\mathcal{F})=d$ such that, for any $\Ccov>1$ and any algorithm such that $\widehat{y}_t$ is some (possibly randomized) function of $\mathcal{H}_{t-1}$ and $X_t$, there is a $\Ccov$-smooth $\nu^2$-subGaussian environment such that, for any horizon $T\geq C[1\vee(\Ccov-1)^{-1}]\nu^2 d$ with some sufficiently large constant $C>0$,
\begin{align*}
\bbE[\reg(T)] \gtrsim T\wedge \sqrt{(\Ccov-1)(1\vee \nu^2)dT}.
\end{align*}

\end{theorem}


\begin{remark}[Comparison with the fully adversarial setting]

When $d\geq 2$, the lower bound in Theorem~\ref{thm:lower-bound-well-specified} matches its counterpart in the fully adversarial setting established in Theorem~\ref{theorem:minimax}, both scaling as $\sqrt{T}$ with the horizon. Thus, for any fixed $\Ccov>1$, imposing well-specification (or even realizability) does not improve the worst-case regret rate. 
The intuition is that, although well-specification rules out arbitrary response sequences, it leaves the adaptive covariate process sufficiently unconstrained to reproduce the same worst-case difficulty, while preserving realizability of the responses.

\end{remark}

\begin{remark}[Optimality of ERM \rev{under full-history smoothness}]

\rev{
Although Theorem~\ref{thm:lower-bound-well-specified} is stated under partial-history smoothness, its lower-bound construction also satisfies the stronger full-history smoothness
(see Appendix~\ref{sec:app_ERM} for the formal definition). 
Our lower bound, together with the upper bound in \citet[Theorem 8]{block2024performance}, establishes the minimax optimality of ERM under full-history smoothness in its dependence on $\nu$, $\Ccov$, $d$, and $T$ for function classes with finite VC dimension, up to logarithmic factors.
}
 To be specific, for any subGaussian environment satisfying full-history smoothness, a minor technical correction to Theorem~8 of \citet{block2024performance} (see Appendix~\ref{sec:correction}) shows that the expected cumulative estimation error of ERM satisfies
\begin{equation}\label{eq:ERM_upper_bound}
\mathbb{E}\left[\sum_{t=1}^T\left(\widehat{f}_t\left(x_t\right)-f^{\star}\left(x_t\right)\right)^2\right] =\widetilde{O}\left(\sqrt{\Ccov (1\vee \nu^2) T\left(1\vee\log \mathbb{E}_\mu\left[W_{2\Ccov T \log T}(256 \cdot \mathcal{F})\right]\right)}\right),
\end{equation}
where $\widetilde{O}$ hides logarithmic factors in $T$ and $\Ccov$, and $W_m(\mathcal{F})$ is the Wills functional defined in Definition~\ref{def:wills} in Appendix~\ref{appendix:ERM}. In the well-specified setting, the expected cumulative error is no greater than the expected regret, \rev{but we show that the same upper bound still holds for the regret in Appendix~\ref{sec:app_error_regret}.} The VC specialization of \eqref{eq:ERM_upper_bound} follows directly from the fact that $\log W_m(\mathcal{F})\lesssim \vc(\mathcal{F})\cdot \log(m)$ for all $m\in\bbN_+$ \citep{block2024performance}. 
\rev{
The full-history assumption cannot generally be relaxed to partial-history smoothness for the ERM upper bound. Appendix~\ref{sec:app_ERM} explains its role in the analysis of \citet{block2024performance} and gives a counterexample under partial-history smoothness in which the unique ERM incurs linear expected cumulative estimation error.
}

\end{remark}


    

Although Theorem~\ref{thm:lower-bound-well-specified} gives a matching worst-case lower bound over function classes, it does not imply that the Wills-functional upper bound in \eqref{eq:ERM_upper_bound} is sharp for every class. The following proposition shows that a function class can have large Wills complexity while every online ERM procedure has expected regret bounded independently of $T$, even when the covariate process is not smooth. 
 
\begin{proposition}
\label{prop:will}
Recall the Wills functional $W_m(\mathcal F)$ from Definition~\ref{def:wills}. For any $K\in\bbN_+$, let $m\geq K\log(2K)$. There exists a function class $\mathcal{F}$ and a base measure $\mu$ on $\mathcal{X}$ with $\log \bbE_{\mu}[W_m(\mathcal{F})] \asymp K$, such that for any horizon $T$ and any $\nu^2$-subGaussian (possibly non-smooth) environment, any empirical risk minimizer as defined in \eqref{eq:ERM} satisfies $\bbE[\reg(T)] \lesssim (1\vee \nu^2)K$. 


\end{proposition}

\section{Conclusion and outlook}
We first study the optimal regret in agnostic smoothed online learning with adversarial responses. By Theorem~\ref{thm:upper-bound}, Algorithm~\ref{alg:meta} achieves regret $\widetilde{O}\big(\sqrt{\operatorname{Pdim}(\mathcal{F})\Ccov T}\big)$, matching the lower bound established in Theorem~\ref{theorem:minimax} up to logarithmic factors. Notably, this lower bound already holds for the linear prediction class $\mathcal{F}_{\textsf{lin}}$ in $\bbR^d$. We then examine empirical risk minimization (ERM) in the well-specified setting as a natural benchmark for oracle-efficient learning. Theorem~\ref{thm:lower-bound-well-specified}, however, shows that even in this setting, ERM cannot attain a regret rate faster than the minimax rate for adversarial responses. Together with the intricate but computationally inefficient \textsc{R-Cover} algorithm of \citet{blanchard2025agnostic}, these findings suggest a possible gap between information-theoretic optimality and computational efficiency. 

\rev{On the other hand, \citet{buzaglo2026oracle} recently considers an oracle-efficient algorithm and its regret bound for binary classification where the covariates are i.i.d. (i.e. $\normalfont{\textsf{C}}_{\textsf{cov}} = 1$ in our notation) and the labels $y_t\in \{0,1\}$ may be adversarial. To summarize, their simple and elegant randomized Follow-the-Perturbed-Leade algorithm \textsc{Gaussian-FTPL} leads to $(T\cdot \textsf{VC}(\mathcal{F}))^{1/2}(\log T)^{5/2}$. The algorithm is appealing from practical aspect as it's computationally much faster and does not require tracking each hypothesis in the function class $\mathcal{F}$. This work motivates us to ask whether adapting randomized FTPL to our setting can yield an oracle-efficient algorithm as a future work, with complexity measured by the number of calls to an empirical risk minimization (ERM) oracle.}

\rev{Our algorithm and analysis assume that the function class has finite pseudo-dimension (see Remark~\ref{remark:pseudo}). We leave open the problem of characterizing optimal regret in the same agnostic setting with adversarial responses under more general conditions such as finite fat-shattering dimension.}

%% file: section/appendix.tex
\section*{Additional notations}
Let $\delta_x$ denote the Dirac measure at $x$. For a function class $\mathcal G$ and $C\subseteq\mathcal X$, write $\mathcal G|_C:=\{(g(x))_{x\in C}:g\in\mathcal G\}$ for the restriction of $\mathcal G$ to $C$.
\section{Deferred proofs}

\rev{
\begin{proof-of-corollary}[\ref{cor:control-strong}]
Let $2\leq L\leq T$ be an integer. We proceed by the following decomposition.
\begin{equation}
\begin{aligned}
    \sum_{1\leq t \leq T}P_t(A_t)
    &= \sum_{1\leq t \leq L}P_t(A_t)\\
    &\quad  + 
    \sum_{L<t\leq T}P_t(A_t\cap\{\sum_{1\leq s <t} \d P_s/\d \mu\leq \C_{\Ccov}(L - 1)\})\\
    &\quad 
      + \sum_{L<t\leq T}P_t(A_t\cap\{\sum_{1\leq s <t} \d P_s/\d \mu> \C_{\Ccov}(L - 1)\}).
\end{aligned}
\end{equation}
The rest of the proof is simply to bound the three terms accordingly and tune $L$ finally.\\
\textbf{Early term:}
On $\textsf{E}$, Proposition~\ref{prop:good-event}(2) controls $\sum_{1\leq s<t}P_s(A_t)$ simultaneously for every $t$. Since $\Ccov\mu-P_s\geq 0$ and has total mass $\Ccov-1$, for $t\geq2$,
\begin{align}
P_t(A_t)&\leq\Ccov\mu(A_t)\\
&=\frac{1}{t-1}\sum_{1\leq s<t}P_s(A_t)
+\frac{1}{t-1}\sum_{1\leq s<t}(\Ccov\mu-P_s)(A_t)\\
&\overset{(\spadesuit)}{\lesssim}\frac{1}{t-1}\Big[1+\log\Big(\delta^{-1}\Pi_{\mathcal{A}}(2T\Ccov\log(T\delta^{-1}))\Big)+\log\Big(\frac{|\mathcal{A}'|}{\delta}\Big)\Big]+\Ccov-1,
\end{align}
where we use Proposition~\ref{prop:good-event} for the inequality~$(\spadesuit)$.
Consequently, by the numeric inequality $\sum_{1\leq t \leq L}t^{-1}\leq \log(2L)$, we obtain that
\begin{equation}
\begin{aligned}
\sum_{1\leq t\leq L}P_t(A_t)
&\lesssim 1+(L-1)(\Ccov-1)\\
&\quad+\log(2L)\Big[1+\log\Big(\delta^{-1}\Pi_{\mathcal{A}}(2T\Ccov\log(T\delta^{-1}))\Big)+\log\Big(\frac{|\mathcal{A}'|}{\delta}\Big)\Big].
\end{aligned}
\end{equation}
For $t>L$, we split $P_t(A_t)$ according to whether $\sum_{1\leq s<t}\d P_s/\d\mu\leq\Ccov(L-1)$. We control the low and high density terms as follows.\\
\textbf{Low density term:}
On this region, the fact that $\d P_t/\d\mu\leq\Ccov$ yields that 
\begin{equation}
\label{eq:low}
\begin{aligned}
&\frac{\d P_t}{\d\mu}\mathbbm{1}\Big\{\sum_{1\leq s\leq t - 1}\frac{\d P_s}{\d\mu}\leq\Ccov(L-1)\Big\}\\
&\leq\Big[\Big(\sum_{1\leq s\leq t}\frac{\d P_s}{\d\mu}\Big)\wedge\Ccov L\Big]
-\Big[\Big(\sum_{1\leq s\leq t-1}\frac{\d P_s}{\d\mu}\Big)\wedge\Ccov L\Big].
\end{aligned}
\end{equation}
Precisely, we have
\begin{itemize}
    \item If $\sum_{1\leq s <t}\d P_s/\d \mu \leq \Ccov (L-1)$, both truncation on the right hand side of the~\eqref{eq:low} is inactive and the bound trivially holds.
    \item Otherwise, the left hand side is zero, which is upper bounded by a non-negative term $\Big[\Big(\sum_{1\leq s\leq t}\frac{\d P_s}{\d\mu}\Big)\wedge\Ccov L\Big] - \Ccov L$.
\end{itemize}
Summing and integrating this inequality~(\ref{eq:low}) yields
\begin{align}
&\sum_{L<t\leq T}P_t\Big(A_t\cap\Big\{\sum_{1\leq s<t}\frac{\d P_s}{\d\mu}\leq\Ccov(L-1)\Big\}\Big)\\
&\leq\int_{\mathcal X}\Big[\Big[\Big(\sum_{1\leq s\leq T}\frac{\d P_s}{\d\mu}\Big)\wedge\Ccov L\Big]
-\sum_{1\leq s\leq L}\frac{\d P_s}{\d\mu}\Big]\d\mu\\
&\leq\int_{\mathcal X}[\Ccov L-\sum_{1\leq s\leq L}\frac{\d P_s}{\d\mu}]\d\mu
=L(\Ccov-1).
\end{align}
\textbf{High density term:}
On the complementary region, smoothness and Proposition~\ref{prop:good-event}(2) give
\begin{align*}
&P_t\Big(A_t\cap\Big\{\sum_{1\leq s<t}\frac{\d P_s}{\d\mu}>\Ccov(L-1)\Big\}\Big)\\
&\leq\frac{1}{L-1}\int_{A_t}\sum_{1\leq s<t}\frac{\d P_s}{\d\mu}\d\mu\\
&=\frac{1}{L-1}\sum_{1\leq s<t}P_s(A_t)\\
&\lesssim\frac{1}{L-1}\Big[1+\log\Big(\delta^{-1}\Pi_{\mathcal{A}}(2T\Ccov\log(T\delta^{-1}))\Big)+\log\Big(\frac{|\mathcal{A}'|}{\delta}\Big)\Big].
\end{align*}
Combining the bounds, we obtain
\begin{equation}
\label{eq:sum-log}
\begin{aligned}
\sum_{1\leq t\leq T}P_t(A_t)
&\lesssim 1+(2L-1)(\Ccov-1)\\
&\quad+\Big(\log(2L)+\frac{T-L}{L-1}\Big)\\
&\qquad\times\Big[1+\log\Big(\delta^{-1}\Pi_{\mathcal{A}}(2T\Ccov\log(T\delta^{-1}))\Big)+\log\Big(\frac{|\mathcal{A}'|}{\delta}\Big)\Big].
\end{aligned}
\end{equation}
Finally, a carefully
optimizing over $L$ to balance the three preceding three terms by setting
\begin{align}
\tag{optimal $L$}
L-1 \asymp \sqrt{\frac{\textsf{A}_{\star} T}{\Ccov -1}}
\end{align}
concludes the proof. We recall from Proposition~\ref{prop:good-event} that, for the sake of convenience
\begin{align*}
    \textsf{A}_{\star}=\log\Big(\delta^{-1} \Pi_{\mathcal{A}}(2T\Ccov \log(T\delta^{-1}))\Big)
    + \log\Big(\frac{|\mathcal{A}'|}{\delta}\Big).
\end{align*} When $\Ccov=1$, take $L=T$.
\end{proof-of-corollary}}


\section{Proof of Theorem~\ref{thm:upper-bound}}
\label{appendix:proof-upper-bound}

Recall the standard definition of the growth function~\citep[Chapter 6]{shalev2014understanding}.
\begin{definition}
\label{def:growth}
Let $\mathcal{G}:\mathcal{X}\to \{0, 1\}$ be a function class. The growth function of $\mathcal{G}$, denoted as $\Pi_{\mathcal{G}}:\bbN^+\to \bbN^+$, is defined as 
\begin{align}
    \Pi_{\mathcal{G}}(n)  =\max_{\substack{C\subseteq \mathcal{X}\\ |\mathcal{C}|=n}}|\mathcal{F}_{C}|.
\end{align}
Namely, $\Pi_{\mathcal{G}}(n)$ is the number of different functions from a set $C$ of size $n$ to $\{0,1\}$ that can be obtained by restricting $\mathcal{H}$ to $C$.
\end{definition}
To control the discrepancy between two functions in the discretized function class, we also introduce the following notion of difference set class
\begin{align}
\label{eq:diff}
    \mathcal{D}=\Big\{D: D=\{f_1(x)\neq f_2(x)\},\forall f_1,f_2\in \mathcal{N}_m\circ \mathcal{F}\Big\}.
\end{align}
The dependency of $\mathcal{D}$ on the resolution of $\mathcal{N}_m$ is stated implicitly.
\begin{lemma}
\label{lemma:growth}
Let $\normalfont{\text{Pdim}}(\mathcal{F}) = d\geq 1$ and $mn\geq d$. Then the following hold:
\begin{enumerate}
    \item $\Pi_{\mathcal{N}_m\circ \mathcal{F}}(n)\leq \Big(\frac{emn}{d}\Big)^{d}$;
    \item 
    $\Pi_{\mathcal{D}}(n)\leq \Big(\frac{emn}{d}\Big)^{2d}$.
\end{enumerate}
\end{lemma}
\begin{proof-of-lemma}[\ref{lemma:growth}] We write the points of the grid in increasing order as
$\mathcal{N}_m=\{a_1<\cdots<a_m\}$ where $a_j = -1 + \frac{2j}{m}$  For $1\leq j<m$, let $b_j = -1 + \frac{2j +1}{m}$. We fix the tie-breaking convention once and for all. Therefore, the value of $\textsf{proj}(u,\mathcal{N}_m)$ is determined by the $m-1$ comparisons between $u$ and $b_1,\ldots,b_{m-1}$. 

As described in Definition~\ref{def:growth},
we fix $C=\{x_1,\ldots,x_n\}\subseteq\mathcal X$.  Every restriction of a function in $\mathcal{N}_m\circ\cF$ to $C$ is determined by the restriction of the subgraph of some $f\in\mathcal{F}$ to $ \widetilde C
  :=\{(x_i,b_j):i\in[n],\ 1\leq j<m\},
 |\widetilde C|=n(m-1).$
Since the subgraph class has VC dimension $d$, the Sauer--Shelah lemma yields, when $n(m-1)\geq d$,
\begin{align*}
 \Big| (\mathcal{N}_m\circ\cF)|_C\Big|
 \leq \sum_{k=0}^{d}\binom{mn}{k}
 \leq \left(\frac{emn}{d}\right)^d
\end{align*}
 This also covers $m=1$, in which case $\mathcal{N}_m\circ\mathcal{F}$ contains only one function.  

For the second statement, again fix $C=\{x_1,\ldots,x_n\}$ and note that the restriction to $C$ of $
  D=\{x:f_1(x)\neq f_2(x)\}\in\cD$
is completely determined by the ordered pair $(f_1|_C,f_2|_C)$.  Consequently, we obtain $|\mathcal{D}|_C|
  \leq |(\mathcal{N}_m\circ\mathcal{F})|_C|^2
  \leq \Pi_{\mathcal{N}_m\circ\cF}(n)^2
  \leq \left(\frac{emn}{d}\right)^{2d}.$
We take the maximum over $C$, completing the proof of the claim.
\end{proof-of-lemma}
The following lemma states a covering number upper estimate essentially. 
\begin{lemma}
\label{lemma:cover1}
For the base measure $\mu$, there exists a subset $\mathcal{M}_1\subseteq \mathcal{N}_m\circ \mathcal{F}$ such that, for any $\varepsilon>0$,
\begin{enumerate}
    \item for every $f\in \mathcal{N}_m\circ \mathcal{F}$, there exists a function $g\in \mathcal{M}_1$ such that $\mu\big(\{x\in \mathcal{X}: f(x)\neq g(x)\}\big)\leq \varepsilon$;
    \item $\log |\mathcal{M}_1|\lesssim d\log(m \varepsilon^{-1})$.
\end{enumerate}
\end{lemma}
\begin{proof-of-lemma}[\ref{lemma:cover1}]
Consider the maximal packing $\mathcal{M}_1=\{g_{i}\}_{1\leq i \leq N|}\subseteq \mathcal{N}_m\circ \mathcal{F}$ such that $\mu\{x:g_i(x)\neq g_{i'}(x)\}>\varepsilon$ for each $i\neq i'$. We first bound the number of such a packing number. We proceed by a probabilistic argument to show the existence of a small-cardinality packing set and lower bound the growth number. Let $Z_1,\cdots,Z_{\texttt{N}}$ be drawn i.i.d. from $\mu $. By independence, for any $\varepsilon$, we obtain that
\begin{align*}
    \bbP\Big(g_i(Z_j)=g_{i'}(Z_j),\forall j\in [\texttt{N}]\Big) = \big(1-\mu\{g_i\neq g_{i'}\}\big)^{\texttt{N}} <(1-\varepsilon)^{\texttt{N}}\leq e^{-\varepsilon\texttt{N}}.
\end{align*}
By a union bound, we obtain that, 
\begin{align*}
    \bbP\Big(\exists i\neq i'\in [N],\forall j\in [\texttt{N}], g_i(Z_j)=g_{i'}(Z_j)\Big)\leq \frac{N(N-1)}{2}e^{-\varepsilon\texttt{N}}<1,
\end{align*}
as long as 
\begin{align}
\label{eq:packing}
    \texttt{N}\geq \frac{5}{\varepsilon}\log N.
\end{align}
Therefore, by considering the complement event, there exists a realization of $Z_1,\cdots,Z_{\texttt{N}}$ on which all $N$ restrictions are different. Therefore, by the definition of the growth function (see Def.~\ref{def:growth}), we obtain that
\begin{align*}
    N\leq \Pi_{\mathcal{N}_m\circ \mathcal{F}}(\texttt{N})\leq \Big(\frac{em\texttt{N}}{d}\Big)^d,
\end{align*}
where we use Lemma~\ref{lemma:cover1}. Therefore, we obtain $\log N \leq d\log(em\texttt{N}d^{-1})\leq d\log(10em(d\varepsilon)^{-1}\log N)$. We obtain that 
\begin{align}
    N\leq \Big(\frac{10em}{\varepsilon}\Big)^{2d}
\end{align}
by the preceding inequality stated above.
\end{proof-of-lemma}
Similar to the preceding lemma, we have the following approximation on $\mathcal{D}$.

\begin{lemma}
\label{lemma:cover2}

For the base measure $\mu$, there exists a sub-class $\mathcal{D}'\subseteq \mathcal{D}$ such that, for any $\varepsilon>0$,
\begin{enumerate}
    \item for every $D\in \mathcal{D}$, there exists a function $D'\in \mathcal{D}'$ such that $\mu\big(D\Delta D'\big)\leq \varepsilon$;
    \item $\log |\mathcal{D}'|\lesssim d\log(m\, \varepsilon^{-1})$.
\end{enumerate}
\end{lemma}

\begin{proof-of-lemma}[\ref{lemma:cover2}]
Recall $\mathcal{D} = \{D:D = \{x\in \mathcal{X}:f_1(x)  \neq f(x)\},\forall f_1,f_2\in \mathcal{N}_m \circ \mathcal{F}\}$. We apply Lemma~\ref{lemma:cover1} by picking $\mathcal{M}_1$ such that, for every $f\in \mathcal{N}_m\circ \mathcal{F}$, there exists a function $g\in \mathcal{M}_1$
such that
\begin{align*}
    \mu\{x\in \mathcal{X}:f(x)\neq g(x)\}\leq \frac{1}{2}\varepsilon.
\end{align*}
We construct $\mathcal{D}'$ by simply enumerating $\mathcal{M}_1:$ $\mathcal{D}' =  \{D':D'=\{x\in \mathcal{X}:g_1\neq g_2(x)\},g_1,g_2\in \mathcal{M}_1\}$. The cardinality is at most $|\mathcal{M}_1|^2$, justifying the second statement. The approximation property is derived by $D =\{f_1(x)\neq f_2(x)\}$ and $D' = \{g_1(x)\neq g_2(x)\}$ where $\mu(\{f_1(x)\neq g_1(x)\})\vee \mu(\{f_2(x)\neq g_2(x)\})\leq \varepsilon/2$, 
\begin{align*}
    \mu\big(D'\Delta D\big)
    & \leq 
    \mu(D'\Delta D \cap \textsf{E}) +\mu(\textsf{E}^{\textsf{c}}) = \mu(\textsf{E}^{\textsf{c}})\leq \varepsilon, 
\end{align*}
where $\textsf{E} = \big\{f_1(x) = g_1(x)\text{ and }f_2(x) = g_2(x)\big\}$.
\end{proof-of-lemma}

\begin{proof-of-proposition}[\ref{prop:good-event}]
By invoking \cite{haghtalab2024smoothed} (Lemma~\ref{lemma:decouple}), if $k \asymp \Ccov \log(4T \delta^{-1})$, we obtain a collection 
\begin{align*}
    \mathcal{Z}=\{Z_{t}^{(j)},t\in [T],j \in [k]\}
\end{align*}
such that, with probability at least $1-\delta/4$, $x_t\in \{z_{t}^{(1)},\cdots,z_t^{(k)}\}$ for each $1\leq t \leq T$ by the numeric inequality $T(1-\Ccov^{-1})^k \leq T \exp(-k \Ccov^{-1})\leq \delta/4$.
\begin{claim}
\label{claim}
Conditional on the event that the decoupling event occurs, the following inequalities hold, for $\varepsilon = \frac{1}{kT}$ and $1< N <T$
\begin{align*}
   ~& \bbP\Big(\exists A_1,A_2\in \mathcal{A}:\mu(A_1\Delta A_2)\leq \varepsilon, \sum_{1\leq t \leq T}\mathbbm{1}\{x_t\in A_1\Delta A_2\}\geq N\Big)\\
    & \leq 
    \mu^{\otimes kT}\Big(\exists A_1,A_2\in \mathcal{A}:\mu(A_1\Delta A_2)\leq \varepsilon,\sum_{1\leq t \leq T}\sum_{1\leq k'\leq k}\mathbbm{1}\{z_t^{(k')}\in A_1\Delta A_2\}\geq N\Big)\\
    &\leq [\Pi_{\mathcal{A}}(2kT)]^2 \exp\{-\textsf{c} N\},
\end{align*}
for some sufficiently small absolute constant $\normalfont{\textsf{c}}>0$.

\end{claim}

We postpone the justification of the claim to the end of the proof. Assuming the claim holds, as long as $N\geq \normalfont{\textsf{c}}^{-1}\log(4[\Pi_{\mathcal{A}}(2kT)]^2\delta^{-1})$, we show the first statement with probability at least $1-\delta/2$. Precisely, with probability at least $1-\frac{1}{2} \delta$, for every pair of sets $A_1,A_2\in \mathcal{A}$ such that $\mu(A_1\Delta A_2)\leq \varepsilon$, we obtain 
\begin{align}
    \sum_{1\leq t \leq T}\mathbbm{1}\Big\{x_t\in A_1\Delta A_2\Big\}\lesssim \log\Big(\delta^{-1} \Pi_{\mathcal{A}}(2T\Ccov \log(T\delta^{-1}))\Big).
\end{align}

It remains to prove Claim~\ref{claim} by rewriting
\begin{align*}
     ~&\mu\Big(\exists A_1,A_2\in \mathcal{A}:\mu(A_1\Delta A_2)\leq \varepsilon,\sum_{1\leq t \leq T}\sum_{1\leq k'\leq k}\mathbbm{1}\{z_t^{(k')}\in A_1\Delta A_2\}\geq N\Big)\\
     & = 
     \mu\Big\{\sup_{\substack{A_1,A_2\in \mathcal{A}\\ \mu(A_1\Delta A_2)\leq \varepsilon}} \sum_{1\leq t \leq T}\sum_{1\leq k'\leq k}\mathbbm{1}\{z_t^{(k')}\in A_1\Delta A_2\}\geq N\Big\}\\
     & \leq 
     \mu\Big\{\sup_{\substack{A_1,A_2\in \mathcal{A}\\ \mu(A_1\Delta A_2)\leq \varepsilon}} \frac{\frac{1}{kT}\sum_{1\leq t \leq T}\sum_{1\leq k'\leq k}\mathbbm{1}\{z_t^{(k')}\in A_1\Delta A_2\} - \mu(z\in A_1\Delta A_2)}{\sqrt{\frac{1}{kT}\sum_{1\leq t \leq T}\sum_{1\leq k'\leq k}\mathbbm{1}\{z_t^{(k')}\in A_1\Delta A_2\} }}\\
     &\quad \geq \frac{(kT)^{-1}N  - \varepsilon}{\sqrt{(kT)^{-1}N}}\Big\}\\
     & \quad\Big[\because u\to \sqrt{u} - \frac{\varepsilon}{\sqrt{u}}\text{ is increasing over }\bbR^+.\Big]\\
     & \leq \Pi_{\mathcal{A}\Delta \mathcal{A}}(2kT) \exp\Big\{-\normalfont{\textsf{c}}\frac{(N - kT\varepsilon)^2}{N}\Big\}\qquad[\because \textbf{Lemma~\ref{lemma:VC-ineq}}]\\
     & \leq [\Pi_{\mathcal{A}}(2kT)]^2 \exp\{-\normalfont{\textsf{c}}N\},
\end{align*}
where the class $\mathcal{A}\Delta \mathcal{A} = \big\{A_1\Delta A_2:A_1,A_2\in \mathcal{A}\big\}$. We consider the case $N-kT\gtrsim \sqrt{N}$ by, specifically, taking $\varepsilon=(kT)^{-1}$. 

We then turn to proving the second event. Let $A\cap \{x_1,\cdots,x_T\} = \emptyset$ and $\mathcal{A}'\subseteq \mathcal{A}$ be the covering such that, for any set $A\subseteq \mathcal{A}_\star$, there exists a set $A'_\star\subseteq \mathcal{A}'_\star$ such that $\mu(A_\star\Delta A'_\star)\leq \varepsilon$. Therefore, we proceed by splitting
\begin{align}
\label{eq:middle}
    \sum_{1\leq t \leq T} P_t(A)
    & \leq \sum_{1\leq t \leq T} P_t(A') + 
    \sum_{1\leq t \leq T}P_t(A\Delta A') \leq  \sum_{1\leq t \leq T} P_t(A') + T \Ccov \varepsilon.
\end{align}
Now we want to show a uniform upper bound of the first term on the left hand side of~\eqref{eq:middle}. We proceed via the following martingale arguments, recently popular in the analysis of time-uniform statistical procedures (for example,~\cite{howard2021time}).

Consider an arbitrary set $\textsf{A}'\in \mathcal{A}'$. For any $\lambda>0$,
\begin{align*}
~&\bbE\Big[\exp(\lambda\cdot P_t(\textsf{A}') - \mathbbm{1}\{x_t\in \textsf{A}'\})\mid \mathcal{H}_{t-1}^x] 
\\
& = 
P_t(\textsf{A}')
\exp(\lambda\cdot P_t(\textsf{A}') - 1)+ (1 -P_t(\textsf{A}')) \exp(\lambda \cdot P_t(\textsf{A}'))\\
& = e^{\lambda P_t(\textsf{A}')}\Big[1 +\Big(\frac{1}{e}-1\Big)P_t(\textsf{A}')\Big]\\
&\overset{(\textcolor{blue}{\dagger})}{\leq } \exp\Big(\lambda P_t(\textsf{A}') + (e^{-1}-1)P_t(\textsf{A}')\Big)\\
& \leq 1,
\end{align*}
if $\lambda\leq 1- e^{-1}$, where the step $(\textcolor{blue}{\dagger})$ is due to $1 + u \leq e^{u}$ for all $u\in \bbR$. Therefore, for sufficiently small $\lambda$, the process 
\begin{align}
    \mathfrak{M}_{t}(\textsf{A}') : =
    \exp\Big(\lambda \sum_{1\leq j \leq t}P_j(\textsf{A}') - \sum_{1\leq j \leq t}\mathbbm{1}\{x_j\in \textsf{A}'\}\Big)
\end{align}
is a nonnegative supermartingale. By Ville's inequality, we obtain that
\begin{align}
     \bbP\Big(\sum_{1\leq j \leq t}P_j(\textsf{A}')\geq \frac{1}{\lambda}\Big[\sum_{1\leq j \leq t} \mathbbm{1}\{x_j\in \textsf{A}'\}  + \zeta\Big]\Big)= \bbP\Big(\rev{\max_{t\leq T}\,}\mathfrak{M}_t(\textsf{A}')>e^{\zeta}\Big)\leq e^{-\zeta}.
\end{align}
By union bounds over all choices of $\textsf{A}'$ and setting $\zeta>\log((4|\mathcal{A}'|)/\delta)$, we obtain that for a fixed time $t$,

\begin{align}
    \sup_{A'_{\star}\in \mathcal{A}'}\Big[\sum_{1\leq t \leq T}P_t(A_\star') - \frac{1}{\lambda}\sum_{1\leq t \leq T}\mathbbm{1}\{x_j\in A_\star'\}\Big]\lesssim \frac{1}{\lambda}\log\Big(\frac{ |\mathcal{A}'|}{\delta}\Big), 
\end{align}
for some $\lambda \in (0,1-e^{-1}]$, with probability at least $1-\delta/4$. Now we are ready to continue from the inequality in~\eqref{eq:middle}:
\begin{align}
    \sum_{1\leq t \leq T}P_t(A)\leq \frac{1}{\lambda}\sum_{1\leq t \leq T}\mathbbm{1}\{x_t \in A'\}+ \frac{\C_{A}}{\lambda}\log\Big(\frac{|\mathcal{A}'|}{\delta}\Big)+ T\Ccov \varepsilon,
\end{align}
for some absolute constant $\C_{A}>0$. If $A\cap \{x_1,\cdots,x_T\}=\emptyset$, then for every $t\in [T]$, $x_t\in A'$ is equivalent to $x_t\in A\Delta A'$. Thus, by combining the inequality stated in~\eqref{eq:event1}, we obtain that
\begin{align}
    \sum_{1\leq t \leq T}P_t(A)\leq \frac{\C_{1}}{\lambda}\log\Big(\delta^{-1} \Pi_{\mathcal{A}}(2T\Ccov \log(T\delta^{-1}))\Big)  +\frac{\C_{A}}{\lambda}\log\Big(\frac{|\mathcal{A}'|}{\delta}\Big)+ T\Ccov \varepsilon.
\end{align}
We then fix $\lambda$, say $0.01$ and obtain tha
\begin{align}
    \sum_{1\leq t \leq T}P_t(A)\lesssim \log\Big(\delta^{-1} \Pi_{\mathcal{A}}(2T\Ccov \log(T\delta^{-1}))\Big)
    + \log\Big(\frac{|\mathcal{A}'|}{\delta}\Big).
\end{align}
\end{proof-of-proposition}

We state a formal version of Theorem~\ref{thm:upper-bound} with a high probability guarantee.

\begin{theorem}[Formal version of Theorem~\ref{thm:upper-bound}]
\label{thm:upper-bound-formal}
For any adversarial environment $\normalfont{\textsf{Env}}$ and time horizon $T\in \bbN^+$, Algorithm~\ref{alg:meta}~\textsc{Hedge-Cover} with discretization \rev{$m\asymp T$} and $\eta\in (0,1/10)$ satisfies, for any $\delta\in (0,1)$,
\begin{equation}
\rev{\begin{aligned}
\bbP\Big(\reg(T,\mathcal F)\lesssim\big[d\log(T)+\sqrt{(\Ccov-1)dT}\big]\big(\log^2(T)+\log(T)\log\delta^{-1}\big)\wedge T\Big)\geq1-\delta.
\end{aligned}}
\end{equation}
\end{theorem}

\begin{proof-of-theorem}[\ref{thm:upper-bound-formal}]
We assume $T\geq \Ccov d$. Otherwise, the statement is trivially true. We invoke Proposition~\ref{prop:good-event} by setting $\mathcal{A}=\mathcal{D}$, defined in~\eqref{eq:diff}. The growth function and cardinality of a covering are upper bounded in Lemma~\ref{lemma:growth} and Lemma~\ref{lemma:cover2}.

Lemma~\ref{lemma:cover1} states that there exists a subset $\mathcal{M}_1=\{f_{1},\cdots,f_{N}\}\subseteq \mathcal{N}_m\circ \mathcal{F}$ such that, for any $\varepsilon>0$ and $\log(N)\lesssim d\log(m \varepsilon^{-1})$. We proceed to consider $\textsf{B}(j,t):=\big\{x\in \mathcal{X}:\sel(\{x_i,f_j(x_i)\}_{1\leq j\leq t-1})(x)\neq f_j(x)\big\}\in \mathcal{D}$ for every $j\in [N]$ and $t\in[T]$. By definition of the selector $\sel$, we have $\textsf{B}(j,t)\cap \{x_1,\cdots,x_{t-1}\}=\emptyset$ for every $j\in [N]$. 


Therefore, by invoking Corollary~\ref{cor:control-strong}, we obtain

\begin{equation}
\rev{\begin{aligned}
\sum_{1\leq t\leq T}P_t\big(\textsf B(j,t)\big)
&\lesssim\log(2T)\Big[1+\log\Big(\delta^{-1}\Pi_{\mathcal{D}}(2T\Ccov\log(T\delta^{-1}))\Big)+\log\Big(\frac{|\mathcal{D}'|}{\delta}\Big)\Big]\\
&\quad+\sqrt{(\Ccov-1)T\Big[1+\log\Big(\delta^{-1}\Pi_{\mathcal{D}}(2T\Ccov\log(T\delta^{-1}))\Big)+\log\Big(\frac{|\mathcal{D}'|}{\delta}\Big)\Big]}.
\end{aligned}}
\end{equation}

For each $j$, the process $
\exp\Big(\sum_{1\leq t \leq n}\mathbbm{1}\{X_t\in\mathsf B(j,t)\}
-(e-1)\sum_{t=1}^nP_t(\mathsf B(j,t))\Big)$
is also a nonnegative supermartingale as 
$e^{-(e-1)P_t(\mathsf B(j,t))}[1+(e-1)P_t(\mathsf B(j,t))]\leq1$.
Thus, with failure probability at most $\delta/4$, uniformly over $j\in[N]$,
\begin{align*}
~&\sum_{t=1}^T\mathbbm{1}\{X_t\in\mathsf B(j,t)\}\\
&\leq(e-1)\sum_{t=1}^TP_t(\mathsf B(j,t))+\log(4N/\delta)\\
& \lesssim 
\rev{\log(2T)\Big[1+\log\Big(\delta^{-1}\Pi_{\mathcal{D}}(2T\Ccov\log(T\delta^{-1}))\Big)+\log\Big(\frac{|\mathcal{D}'|}{\delta}\Big)\Big]}\\
&\quad\rev{+\sqrt{(\Ccov-1)T\Big[1+\log\Big(\delta^{-1}\Pi_{\mathcal{D}}(2T\Ccov\log(T\delta^{-1}))\Big)+\log\Big(\frac{|\mathcal{D}'|}{\delta}\Big)\Big]}+\log(4N/\delta),}
\end{align*}
with probability at least $1-\delta/2$.
\rev{
Since $\emptyset\in\mathcal D$, Proposition~\ref{prop:good-event}(1) controls the empirical disagreements between each discretized function and its representative in $\mathcal M_1$. Their squared losses differ by at most $4$ per disagreement, and discretization contributes $O(T/m)=O(1)$. Combining these bounds with Lemma~\ref{lemma:hedge} and Lemmata~\ref{lemma:growth}--\ref{lemma:cover2} gives
\begin{align}
\reg(T,\mathcal F)\lesssim \Big\{\big[d\log(2T)+\sqrt{(\Ccov-1)dT}\big]\log(2T)\big(\log(2T)+\log\delta^{-1}\big)\Big\}\wedge T
\end{align}}

\end{proof-of-theorem}

\section{\textsc{Hedge} details}
\label{appendix:alg}

The following subroutines use the expert class, prior, and cumulative losses stated in Algorithm~\ref{alg:meta}. 

\begin{algorithm}[H]
\LinesNumbered
\caption{\textsc{Hedge-update}}
\SetNlSty{textbf}{}{}
\label{alg:hedge_update}
\KwIn{Response $y_t$; stored predictions $\{\widehat y_{t,\bm h}\}_{\bm h\in\mathcal{H}}$ and losses $\{L_{t-1,\bm h}\}_{\bm h\in\mathcal{H}}$}
\For{$\bm h\in\mathcal{H}$}{
$L_{t,\bm h}\gets L_{t-1,\bm h}+(y_t-\widehat y_{t,\bm h})^2$\;
}
\Return $\{L_{t,\bm h}\}_{\bm h\in\mathcal{H}}$\;
\end{algorithm}

\begin{algorithm}[H]
\LinesNumbered
\caption{\textsc{Hedge-prediction}}
\SetNlSty{textbf}{}{}
\label{alg:hedge_pred}
\KwIn{Predictions $\{\widehat y_{t,\bm h}\}_{\bm h\in\mathcal{H}}$, step size $\eta>0$; stored prior and cumulative losses}
\For{$\bm h\in\mathcal{H}$}{
$\pi_{\bm h}\gets p_{0,\bm h}\exp(-\eta L_{t-1,\bm h})$\;
}
\For{$\bm h\in\mathcal{H}$}{
$\displaystyle p_{t,\bm h}\gets
\frac{\pi_{\bm h}}{\sum_{\bm h'\in\mathcal{H}}\pi_{\bm h'}}$\;
}
$\displaystyle\widehat y_t\gets
\sum_{\bm h\in\mathcal{H}}p_{t,\bm h}\widehat y_{t,\bm h}$\;
\Return $\widehat y_t$\;
\end{algorithm}

\begin{lemma}
\label{lemma:hedge}
Suppose $y_t,\widehat y_{t,\bm h}\in[-1,1]$ for every $t\in[T]$ and $\bm h\in\mathcal{H}$, and let $0<\eta\leq1/8$. With the prior in Algorithm~\ref{alg:meta}, the two routines above satisfy, pathwise and simultaneously for every $\bm h\in\mathcal{H}$,
\begin{align}
&\sum_{t=1}^T(\widehat y_t-y_t)^2
 -\sum_{t=1}^T(\widehat y_{t,\bm h}-y_t)^2
 \leq\frac{1}{\eta}\log\left(\frac{1}{p_{0,\bm h}}\right),\\
\label{eq:num}
&\log\left(\frac{1}{p_{0,\bm h}}\right)
 \leq1+2\log(|\mathcal{T}_{\bm h}|+1)
 +|\mathcal{T}_{\bm h}|
   \log\left(\frac{eT(m+1)}{|\mathcal{T}_{\bm h}|}\right).
\end{align}
The last term in~\eqref{eq:num} is defined to be zero when $\mathcal{T}_{\bm h}=\emptyset$. In particular, the inequalities also apply to an expert selected after the complete trajectory has been observed.
\end{lemma}

\begin{proof-of-lemma}[\ref{lemma:hedge}]
For fixed $y\in[-1,1]$, put $\Phi_y(z)=\exp(-\eta(z-y)^2)$. Its second derivative with respect to $z$ is
\begin{align*}
\Phi_y''(z)
=\Phi_y(z)\big(4\eta^2(z-y)^2-2\eta\big)\leq0,
\qquad z\in[-1,1],
\end{align*}
because $(z-y)^2\leq4$ and $\eta\leq1/8$. Thus $\Phi_y$ is concave on the prediction interval.

Write
\begin{align*}
\textsf{W}_t
 =\sum_{\bm h\in\mathcal{H}}p_{0,\bm h}
        e^{-\eta L_{t,\bm h}},
\qquad
L_{t,\bm h}=\sum_{i=1}^t(\widehat y_{i,\bm h}-y_i)^2.
\end{align*}
Since the prior is a probability distribution and $L_{0,\bm h}=0$, we have $\textsf{W}_0=1$. Algorithm~\ref{alg:hedge_pred} uses
$p_{t,\bm h}=p_{0,\bm h}e^{-\eta L_{t-1,\bm h}}/\textsf{W}_{t-1}$.
Jensen's inequality therefore yields
\begin{align*}
e^{-\eta(\widehat y_t-y_t)^2}
&=\Phi_{y_t}\left(\sum_{\bm h\in\mathcal{H}}
                         p_{t,\bm h}\widehat y_{t,\bm h}\right)\\
&\geq\sum_{\bm h\in\mathcal{H}}
         p_{t,\bm h}e^{-\eta(\widehat y_{t,\bm h}-y_t)^2}\\
&=\frac{\sum_{\bm h\in\mathcal{H}}
         p_{0,\bm h}e^{-\eta L_{t-1,\bm h}}
                    e^{-\eta(\widehat y_{t,\bm h}-y_t)^2}}
        {\textsf{W}_{t-1}}\\
&=\frac{\textsf{W}_t}{\textsf{W}_{t-1}}.
\end{align*}
Multiply over $t\in[T]$ and take logarithms to obtain
\begin{align*}
\sum_{t=1}^T(\widehat y_t-y_t)^2
\leq-\eta^{-1}\log\textsf{W}_T.
\end{align*}
For every expert, $\textsf{W}_T\geq p_{0,\bm h}e^{-\eta L_{T,\bm h}}$, we substitute this lower bound proves the first inequality simultaneously for all experts.

It remains to evaluate the initializiation. There are exactly $\binom Tk m^k$ experts with $|\mathcal{T}_{\bm h}|=k$: the set of correction times has $\binom Tk$ possibilities, and each map from those times to the $m$-point grid has $m^k$ possibilities. Hence the prior in Algorithm~\ref{alg:meta} is precisely
\begin{align*}
p_{0,\bm h}
&=\frac{1}{\texttt{p}}
  \frac{1}{(|\mathcal{T}_{\bm h}|+1)^2
           (m+1)^{|\mathcal{T}_{\bm h}|}
           \binom T{|\mathcal{T}_{\bm h}|}},\\
\texttt{p}
&=\sum_{k=0}^T\frac{(m/(m+1))^k}{(k+1)^2}.
\end{align*}
The $k=0$ term equals one, and
$1\leq\texttt{p}\leq\sum_{k=0}^{\infty}(k+1)^{-2}\leq2$.
Consequently, we have
\begin{align*}
\log\left(\frac1{p_{0,\bm h}}\right)
={}&\log\texttt{p}+2\log(|\mathcal{T}_{\bm h}|+1)\\
&+|\mathcal{T}_{\bm h}|\log(m+1)
 +\log\binom T{|\mathcal{T}_{\bm h}|}.
\end{align*}
For $1\leq|\mathcal{T}_{\bm h}|\leq T$, use
$\binom T{|\mathcal{T}_{\bm h}|}\leq
 (eT/|\mathcal{T}_{\bm h}|)^{|\mathcal{T}_{\bm h}|}$
and $\log\texttt{p}\leq\log2\leq1$ to prove~\eqref{eq:num}.
\end{proof-of-lemma}
\section{Proof for Section~\ref{sec:lower-bound}}
\label{appendix:lower-bound}

\subsection{Proof of Proposition~\ref{prop:upper-unregularized-VAW}}

\begin{proof-of-proposition}[\ref{prop:upper-unregularized-VAW}]
We isolate the step where our improvement arises. All other parts of the proof proceed exactly as in \citet[Appendix C]{chen2026self}. 
Recall that, for $r\in(0,1]$, we call a subsequence $(x_{i_1},\ldots,x_{i_k})$ $r$-bad when
\begin{align*}
x_{i_j}^\T\Big(\sum_{\ell\leq j}x_{i_\ell}x_{i_\ell}^\T\Big)^\dagger
x_{i_j}\geq r,
\end{align*}
for $j\in[k]$ and $N(r,\bm x)$ is defined as the longest such subsequence.
The key difference lies in the high-probability control of $N(r,\bm x)$, where $\bm x=(x_1,\cdots,x_T)$.
\begin{itemize}
    \item Chen, Qian, Rakhlin, and Zhivotovskiy apply the decoupling lemma to relate the adaptive covariate sequence $\{x_t\}_{1\leq t \leq T}$ to a $\Theta(T\log(T))$-length i.i.d. sequence drawn from $\mu$, denoted by $\bm z$. For $r>0$ and $1\leq k \leq T$, there exists $K\geq \C_{\textsf{cov}}\log(T/\delta)$ such that, for every $t\in [T]$, $x_{t} \in \{z_{t,K}\}$, with probability $1-\delta$. They then show that, if $K\geq \C_{\textsf{cov}} \log(2T/\delta)$ and $\mathcal{I}:= \{2^{-i}\}_{1\leq i \leq \log_2 (TK)}$, then
    \begin{align}
        \bbP\Big(\exists r_i\in \mathcal{I},N(r_i,\bm x)>k_i\Big)
         & \leq \mu(
         \exists r_i\in \mathcal{I},N(r_i,\bm z)>k_i) + \delta/2\leq \delta,
    \end{align}
    provided that
    \begin{align}
    \label{eq:ki_choice}
        k_i = \Omega\Big(\sqrt{\frac{dTK}{r_i}} + \log\log(TK \delta^{-1})\Big).
    \end{align}

    \item We proceed slightly differently. By Definition~\ref{def:smoothness}, for any subsequence $I = [i_1,\cdots,i_k]\subseteq [T]$ and $\bm x_{I} = (x_{i_1},\cdots,x_{i_k})$, and for any measurable subset $\mathcal{A}\subseteq \bbR^{dk}$,
    \begin{align}
        \bbP(\bm x_{I}\in \mathcal{A})\leq \C_{\textsf{cov}}^k\cdot \mu_{\otimes k}(\mathcal{A}).
    \end{align}
    Conditional on a
remaining multiset of size $j$, the average leverage is its rank divided by $j$,
at most $d/j$. Markov's inequality bounds the chance that the next removed vector
has leverage at least $r$ by $1\wedge d/(rj)$. Therefore, we obtain that
\begin{align}
\mu^{\otimes k}\{\text{$r$-bad}\}
\leq\prod_{j=1}^k\min\{1,d/(rj)\}
\leq\frac{(d/r)^k}{k!}\leq\left(\frac{ed}{rk}\right)^k.
\end{align}
Lemma~\ref{lemma:prod} and a union bound yield that
\begin{equation}\label{eq:badprob}
\bbP(N(r,\bm x)\geq k)
\leq\Ccov^k\binom Tk\left(\frac{ed}{rk}\right)^k
\leq\left(\frac{e^2\Ccov dT}{rk^2}\right)^k.
\end{equation}
Thus $k=\lceil2e\sqrt{\Ccov dT/r}+\log(1/\delta)\rceil$ suffices.

Take $r_i=2^{-i}$, $i\geq1$, and apply \eqref{eq:badprob} with failure probability
$\delta2^{-i}$. Their sum is $\delta$, so simultaneously
\begin{align*}
N(2^{-i},\bm x)\leq2e\sqrt{\Ccov dT\,2^i}
+\log(1/\delta)+i\log2+1.
\end{align*}
Integration and a dyadic partition yield 
\begin{align*}
\sum_{t=1}^T x_t^\T\V_t^\dagger x_t
&\leq\int_0^1N(r,\bm x)\,\d r
\leq\sum_{i=1}^{\infty}2^{-i}N(2^{-i},\bm x)\\
&\lesssim (1+\sqrt2)\sqrt{\Ccov dT}
+\log(1/\delta).
\end{align*}
This proves Proposition~\ref{prop:upper-unregularized-VAW}. 
\end{itemize}
\end{proof-of-proposition}
\begin{lemma}
\label{lemma:prod}
For every 
$I\subseteq [T]$ with $I = (i_1,\cdots,i_k)$, the joint law $\Q_I$ of
$(x_{i_1},\ldots,x_{i_k})$ satisfies $
 \bbQ_I\leq \C_{\textsf{cov}}^k\mu_{\otimes k}.$
Precisely, the inequality holds on every measurable subset of $(\bbR^d)^k$.
\end{lemma}

\begin{proof}
Let $f:(\bbR^d)^k\to[0,\infty]$ be measurable.  Condition first on
$\mathcal F^X_{i_k-1}$ and use the conditional kernel form stated in Definition~\ref{def:smoothness}.  We obtain
\begin{align*}
 \bbE f(X_{i_1},\ldots,X_{i_k})
 &\leq \C_{\textsf{cov}}\,
 \bbE\int f(X_{i_1},\ldots,X_{i_{k-1}},z_k)\,\mu(\d z_k).
\end{align*}
The integrand on the right now depends only on the first $k-1$ selected
variables.  Repeating the same argument at times
$i_{k-1},\ldots,i_1$ gives
\begin{align}
    \bbE f(X_{i_1},\ldots,X_{i_k})
 \leq \C_{\textsf{cov}}^k
 \int f(z_1,\ldots,z_k)\,\mu(\d z_1)\cdots\mu(\d z_k).
\end{align}
For complete kernel justification, the argument first applies to
nonnegative simple $f$ and then to general nonnegative $f$ by monotone
convergence. Taking $f$ to be an indicator completes the proof of the lemma.
\end{proof}

\subsection{Proof of Theorem~\ref{theorem:minimax}}
\begin{lemma}[Theorem 7 in~\citet{chen2026self}.]
\label{lemma:regret-depth}
Let $T\geq 1$ and when $d\geq 2$, for any $\varepsilon\in (0,1)$, for any learner $\pi$, there exists a process $X_1,\cdots,X_T\in 
\bbR^d$ such that 
\begin{align}
    \bbE_{\pi,\normalfont{\textsf{Env}}}\big(\reg(T,\mathcal{F}_{\textsf{lin}})\big)\geq (1-\varepsilon^2)T
\end{align}
\end{lemma}

\begin{remark}
The two dimensional case in Lemma~\ref{lemma:regret-depth} is, in fact, sufficient for our purpose, as in the proof of Theorem~\ref{theorem:minimax}, we only need $d/2$ orthogonal planes in order to implement the hard yet $\textsf{C}_{\textsf{cov}}$-smooth environment. We take $\varepsilon=1/\sqrt{2}$ instead of arbitrarily small $\varepsilon>0$. A more careful hard case design may encode $\varepsilon$ to the lower bound to attain a sharper constant in Theorem~\ref{theorem:minimax}.
\end{remark}
\begin{corollary}
\label{cor:depth-to-regret}
Let $d' = \lfloor d/2\rfloor$ copies of the predictable tree in
Lemma~\ref{lemma:regret-depth} be embedded in mutually orthogonal planes, 
two-dimensional coordinate subspaces of $\bbR^d$. Let $N_t(D)$ be defined as \eqref{eq:realized_depth}, which
denotes
the actual depth after $t$ rounds in the $D$-th plane $D\in [d']$. Then every online learner satisfies
\begin{align}
\mathbb{E}\big(\reg(T,\mathcal{F}_{\textsf{lin}})\big)
\geq
\frac{1}{2}
\sum_{1\leq D\leq d'}\mathbb{E} (N_T(D)\wedge K)
\end{align}
\end{corollary}
\begin{proof-of-corollary}[\ref{cor:depth-to-regret}]
Let $
N'_T(D):=N_T(D)\wedge K$.
For the $D$-th plane and $ \varepsilon_{D,k}\overset{i.i.d.}{\sim}\textsf{Rad}$, for $1\leq k \leq K,1\leq D\leq d'$ recall
\begin{align}
\S_{N'_T(D),D}
:=
\sum_{1\leq k \leq N_T'(D)}\varepsilon_{k,D}x_{D,k},
\qquad
\V_{N'_T(D),D}
:=
\sum_{1\leq k \leq N_T'(D)}x_{D,k}x_{D,k}^{\T},
\end{align}
where $x_{D,k}$ and $\varepsilon_{D,k}$ denote the covariates and response revealed in Lemma~\ref{lemma:regret-depth}.
Lemma~\ref{lemma:regret-depth}, applied to the realized prefix of
length \(N'_T(D)\), gives
\begin{align}
\S_{N'_T(D),D}^{\T}
\V_{N'_T(D),D}^{\dagger}
\S_{N'_T(D),D}
\geq \frac{N'_T(D)}{2}.
\end{align}

Take the first two nodes to be the two coordinate unit vectors in that plane.
For $k\geq3$, choose $x_{D,k}$ predictably so that
\begin{equation}\label{eq:treeorth}
x_{D,k}^{\T}\V_{k-1,D}^{-1}\S_{k-1,D}=0,
\qquad
x_{D,k}^{\T}\V_{k-1,D}^{-1}x_{D,k}=1.
\end{equation}
Here the inverse is taken within the two-dimensional plane. It exists after the first
two nodes. A nonzero vector orthogonal to $\V_{k-1,D}^{-1}\S_{k-1,D}$ exists;
rescale it to obtain the second equality. Fix a deterministic choice whenever there is
more than one possibility. Thus the node depends only on earlier signs.

The Sherman-Morrison identity provides the following identity:
\begin{align*}
&\S_{k,D}^{\T}\V_{k,D}^{-1}\S_{k,D}
-\S_{k-1,D}^{\T}\V_{k-1,D}^{-1}\S_{k-1,D}\\
&\quad=
\frac{x_{D,k}^{\T}\V_{k-1,D}^{-1}x_{D,k}
 -(x_{D,k}^{\T}\V_{k-1,D}^{-1}\S_{k-1,D})^2
 +2\varepsilon_{D,k}x_{D,k}^{\T}\V_{k-1,D}^{-1}\S_{k-1,D}}
 {1+x_{D,k}^{\T}\V_{k-1,D}^{-1}x_{D,k}}
=\frac12.
\end{align*}
At depths one and two the self-normalized quantity is respectively $1$ and $2$.
Consequently, for every path and every $0\leq n\leq K$,
\begin{equation}\label{eq:prefix}
\S_{n,D}^{\T}\V_{n,D}^{\dagger}\S_{n,D}\geq\frac n2,
\qquad
\inf_{\theta_D\in\mathsf H_D}
\sum_{k=1}^n(\theta_D^{\T}x_{D,k}-\varepsilon_{D,k})^2\leq\frac n2.
\end{equation}
\end{proof-of-corollary}

\begin{proof-of-theorem}[~\ref{theorem:minimax}]
We first show the multivariate case as follows.

\underline{$d\geq 2$ and $\C_{\textsf{cov}}>1$.
} 
Let $d' = \lfloor d/2\rfloor$ be the halved dimension.
We first sample $K$ independent $\{\bm \varepsilon_k\}_{1\leq k \leq K}\subseteq \bbR^{d'}$ Radamacher vector of dimension $d'$. Namely, each coordinate $\varepsilon_{D,k}\overset{i.i.d.}{\sim}\textsf{Rad}$. We will fix the choice of $K$ and $p\in (0,1)$ shortly. Now we turn to constructing the base measure $\mu$ and $\textsf{Env}$ to satisfy the definition of $\Ccov$-smoothness. For fixed $\bm \varepsilon\in \bbR^{d'\times K}$, the environment is constructed as follows.

\textbf{Environment $P_t$:}
\begin{enumerate}
    \item With probability $1-p$, set $D_t =0$ and reveal $(x_t,y_t) = 0$
    \item With probability $p$, set $D_t =1$ and sample $J_t\sim \textsf{unif}[1,2,\cdots,d']$. 
    \begin{itemize}
        \item 
    If the coordinate which $J_t$ selects at round $t$ reaches $K$, namely if $J_t = D$ and $N_t(D)> K$,
    where
    \begin{align}\label{eq:realized_depth}
        N_t(D):=\sum_{1\leq i \leq t} \mathbbm{1}\{D_i = 1,J_i = D\} ,
    \end{align}
    also reveal $(x_t,y_t) =0$.
    \item Otherwise, if $N_t(D) = \sum_{1\leq i \leq t} \mathbbm{1}\{J_i = D\} \leq K-1$, then reveal that the covariate of round $t$: $x_t$ belongs to the binary tree in the $J_t$-th plane and the node corresponding to $(\varepsilon_{D,1},\cdots,\varepsilon_{D,N_t})$. Precisely, let 
    \begin{align*}
        (X_t,Y_t) = \big(x_{J_t,N_t(J_t
        )}(\varepsilon_{J_t,1},\cdots,\varepsilon_{J_t,N_t(J_t)-1},\varepsilon_{J_t,k}),\varepsilon_{J_t,N_t(J_t)}\big).
    \end{align*}
    \end{itemize}
\end{enumerate}
We pause and comment on the construction briefly. Notably, the environment is well-defined as long as there are $d'$ different trees of depth $K$, located in $d'$ orthogonal planes. The exact implementation of the orthogonal planes is minor. To be specific, they can be $\textsf{H}_{D}=\{(0,\cdots,x_{2D-1},x_{2D},\cdots,0)\in \bbR^{d}\}$.
The constructions are carefully designed so that each $x_t$ lies on different straight lines except $0$.

Let the following discrete measure be the base measure.  
\begin{align}
    \mu_{\bm \varepsilon} = 
    \frac{1}{\C_{\textsf{cov}}}\frac{p}{d'} \sum_{1\leq D \leq d'}\sum_{1\leq k\leq K}\delta_{x_{D,k}}  + \Big(1 - \frac{Kp}{\textsf{C}_{\textsf{cov}}}\Big)\delta_{\bm 0}.
\end{align}
We now show that the base measure is indeed rendering the $\C_{\textsf{cov}}$-smoothness. For each round $t$ and $x_t = x_{D,k}\neq 0$, we have
\begin{align*}
    P_t(X_t = x_{D,k}\mid \mathcal{H}_{t-1}^{x}) & = \frac{p}{d'} \mathbbm{1}\{N_t(D)\leq K\}\\
   & \leq \C_{\textsf{cov}}\frac{p}{\C_{\textsf{cov}} d'}  = \C_{\textsf{cov}}\cdot \mu_{\bm \varepsilon}.
\end{align*}
If $x_{t} = 0$, it suffices to choose $K\in \bbN^+$ and $p\in (0,1)$ such that
\begin{align}
\label{eq:condition1}
    \C_{\textsf{cov}} \mu_{\varepsilon}(\{0\}) = \C_{\textsf{cov}}  -Kp\geq 1.
\end{align}
As long as~\eqref{eq:condition1} is satisfied, the smoothness is proved for every possible realization of $\bm \varepsilon \in \{\pm 1\}^{d'\times K}$. To proceed with the analysis of regret, suppose the environment invokes $D$-th tree at round $t$,
\begin{equation}
\label{eq:op}
\begin{aligned}
~&\bbE_{\pi,\textsf{Env}}\Big[(\widehat{y}_t - y_t)^2\mid x_1,y_1,\cdots,x_{t-1},y_{t-1},x_{t-1}\Big]\\
&= 
    \bbE_{\pi,\bm \varepsilon}\Big[(\widehat{y}_t - \varepsilon_{D,N_t(D)})^2\mid x_1,y_1,\cdots,x_{t-1},y_{t-1},x_{t-1}\Big] \\
    & = 1+\widehat{y}_t^2\geq 1,
\end{aligned}
\end{equation}
where we use the fact that $\varepsilon_{D,N_t(D)}\sim \textsf{Rad}$ and is independent from all the previous history. To make sure that the tree can be finished (i.e., $N_d(T)=K$) with non-negligible probability, we choose $p\in (0,1)$ such that
\begin{align}
\label{eq:condition2}
    \bbE N_T(D) = p\cdot \frac{T}{d'}\geq 5 K.
\end{align}
Combining~\eqref{eq:condition1} and~\eqref{eq:condition2}
leads to the choice that
\begin{align*}
   K =\Big\lfloor\min\Big\{\sqrt{\frac{(\C_{\textsf{cov}}-1)T}{5 d'}},\frac{T}{10d'}\Big\}\Big\rfloor,\quad p = \frac{5d' K}{T}\in (0,1/2).
\end{align*}
Now we invoke the Paley-Zygmund inequality to show that 
\begin{align*}
    \bbP( N_T(D)\geq K)&\geq \bbP\Big(N_T(D)\geq \frac{1}{5}\bbE [N_T(D)]\Big) \\
      &\geq \frac{16}{25} \cdot \frac{1}{1 + \textsf{Var}(N_T(D))/\bbE(N_T(D))^2}\\
    & \geq \frac{16}{25} \frac{\bbE(N_T(D))}{1  + \bbE(N_T(D))}\\
    & \geq \frac{8}{15}.
\end{align*}
Note that we implicitly use $T\gg d'(\textsf{C}_{\textsf{cov}}-1)$. This implies that
\begin{align*}
    \sum_{1\leq D\leq d'}
    \bbE(N_T(D)\wedge K)  & = 
    \sum_{1\leq D\leq d'}
\Big[\bbE(N_T(D)\mathbbm{1}\{N_T(D)\leq K\})
     + 
     K \bbP(N_T(D)> K)\Big]\\
     &\geq  \frac{8}{15}Kd'.
\end{align*}
Now we relate the regret to the preceding display by Corollary~\ref{cor:depth-to-regret}:
\begin{align}
\mathbb{E}_{\textsf{Env}(\bm \varepsilon)}\big(\reg(T)\big)
\geq
\frac{1}{2}
\sum_{1\leq D\leq d'}\mathbb{E}_{\textsf{Env}(\bm \varepsilon)} (N_T(D)\wedge K)\geq \frac{4}{15} Kd',
\end{align}
where the subscript $\bbE_{\textsf{Env}(\bm \varepsilon)}$ takes randomness over the laws of Bernoulli $\textsf{Ber}(p)$ and  $(\textsf{unif}[d'])^{\otimes T}$, given $\bm \varepsilon$.

Therefore, for sufficiently large $T$, we have for any possible randomized learner $\pi$
\begin{align*}
~&\bbE\big[\reg(T,\mathcal{F}_{\textsf{lin}})\big]\\
& = 
\bbE\Big[\sum_{1\leq t \leq T}(\widehat{y}_t - y_t)^2 - \inf_{f\in \mathcal{F}_{\textsf{lin}}}(f(x_t)-y_t)^2\Big]\\
&\overset{(\spadesuit)}{\geq}
\bbE\Big[\sum_{1\leq D\leq d'}\Big[\sum_{1\leq k \leq N_T(D)\wedge K}(\widehat{y}_{\tau_{D,k}}-\varepsilon_{D,k})^2 - \inf_{\theta_{D}\in \textsf{H}_D}
\sum_{1\leq k \leq N_T(D)\wedge K}(\theta_{D}^{\T}x_{D,k}-\varepsilon_{D,k})^2
\Big]\Big]\\
&\overset{(\heartsuit)}{\geq }
\bbE\Big[\sum_{1\leq d'\leq D}\Big[\sum_{1\leq k \leq N_T(D)\wedge K}(\widehat{y}_{\tau_{D,k}}-\varepsilon_{D,k})^2 - \frac{N_T(D)\wedge K}{2}\Big]\Big]\\
&\overset{(\clubsuit)}{\geq } \frac{1}{2}\sum_{1\leq D\leq d'}\bbE[N_T(D)\wedge K]\\
&\overset{(\diamondsuit)}{\geq }\frac{4}{15}Kd'.
\end{align*}
We justify the preceding four inequalities as follows. Regarding the inequality $(\spadesuit)$ for each round, the regret is non-negative if and only if $D_t = 1$ for any linear predictor. We apply Corollary~\ref{cor:depth-to-regret} and inequality~\ref{eq:op} for the inequalities ($\heartsuit$) and ($\clubsuit$) respectively. For the last step $(\diamondsuit)$, the Paley-Zygmund inequality step above finishes the proof.

For any learner $\pi$, we obtain that
\begin{align}
\bbE\big[\reg(T,\mathcal{F}_{\textsf{lin}})\big]\gtrsim Kd' \asymp \sqrt{d'T(\C_{\textsf{cov}}-1)}.
\end{align}

\underline{$\Ccov=1$.} The desired lower bound follows from the construction underlying Theorem~4 in \citet{gaillard2019uniform}. Their covariates are sampled i.i.d. uniformly from the $d$ standard basis vectors, so this environment satisfies $\Ccov=1$ and is admissible in our setting. Their bound holds even when the learner knows all covariates beforehand, so withholding future covariates can only increase the minimax regret. They consider the unrestricted linear comparator class, but since the responses lie in $[-1,1]$, clipping the linear comparators to this interval can only decrease the comparator loss and hence increase the regret. Taking a fixed Beta prior parameter in their proof yields the desired expected regret lower bound $\gtrsim d\log(1+T/d)$.

\underline{$d=1$ and $\C_{\textsf{cov}}>1$.} Fix an integer $10\leq K\leq T$ and $p\in(0,1)$. By Proposition~6 in~\citet{chen2026self}, there exists a one-dimensional dyadic martingale $(X_k^0 = X_k^0(\varepsilon_1,\ldots, \varepsilon_{k-1}))_{k\in[K]}$ such that
$$
\bbE_{\varepsilon}[R_K]:=\bbE_{\varepsilon}\bigg[\frac{(\sum_{k=1}^K\varepsilon_kX_k^0)^2}{\sum_{k=1}^K(X_k^0)^2}\bigg]
\geq c_0\log T
$$ with some universal constant $c_0>0$. For every fixed path $\bm e=(e_1,\ldots,e_K)\in\{\pm 1\}^K$, define $x_k^0(\bm e):= X_k^0(e_1,\ldots,e_{k-1})$. 

We now construct $\textsf{Env}(\bm e)$ and a base measure $\mu_{\bm e}$ with respect to which $\textsf{Env}(\bm e)$ is $\textsf{C}_{\textsf{cov}}$-smooth,
following a simplified version of the construction used in the multivariate case.
At each round $t$, sample $Z_t\sim\textsf{Bernoulli}(p)$ independent of everything else, and define $N_t:= \sum_{1 \leq i\leq t}Z_t$.
\begin{itemize}
\item If $Z_t=0$ or $N_t>K$, reveal $(x_t,y_t) = 0$.
\item Otherwise, reveal  $(x_t,y_t) = (x_{N_t}^0(\bm e),e_{N_t})$.
\end{itemize}

Assume that $\C_{\textsf{cov}}  -Kp\geq 1$. We let the following counting measure $\mu_{\bm e}$ be the base measure. 
\begin{align}
    \mu_{\bm e} = 
    \frac{p}{\C_{\textsf{cov}}} \sum_{1\leq k\leq K}\delta_{x_k^0(\bm e)}  + \Big(1 - \frac{Kp}{\textsf{C}_{\textsf{cov}}}\Big)\delta_{0}.
\end{align}

To see that $\textsf{Env}(\bm e)$ is $\C_{\textsf{cov}}$-smooth with respect to $\mu_{\bm e}$, note that for each round $t$ and $k\in[K]$, we have
$$
\bbP_t(X_t = x_k^0(\bm e)\mid \mathcal{H}^x_{t-1})\leq p\leq \Ccov \cdot\mu_{\bm e}(x_k^0(\bm e)),
$$
and 
$$
\bbP_t(X_t = 0\mid \mathcal{H}^x_{t-1})\leq1\leq \Ccov -pK\leq \Ccov \cdot\mu_{\bm e}(0).
$$

Next, we derive a lower bound for $\bbE_{\bm \varepsilon, \textsf{Env}(\bm \varepsilon), \pi}[\reg(T)]$. Let $(\tau_{k})_{k=1}^{N_T\wedge K}$ be the activated rounds where $Z_{\tau_k}=1$. Note that $\bm \varepsilon$ and $N_T$ are independent, we have
\begin{equation*}
\begin{aligned}
&\bbE_{\bm \varepsilon, \textsf{Env}(\bm \varepsilon), \pi}[\reg(T)\mid N_T\wedge K = m] \\
\geq & 
\bbE_{\bm \varepsilon, \textsf{Env}(\bm \varepsilon), \pi}\bigg[
\sum_{k=1}^{N_T\wedge K}(\widehat{y}_{\tau_k}-\varepsilon_k)^2 - \inf_{\theta}\sum_{k=1}^{N_T\wedge K} (\langle \theta, x_k^0(\bm \varepsilon) - \varepsilon_k)^2
\mid N_T\wedge K = m\bigg]\\
\geq & \bbE_{\bm \varepsilon}[R_{N_T\wedge K}\mid N_T\wedge K = m].
\end{aligned}
\end{equation*}
Thus, 
$$
\bbE_{\bm \varepsilon, \textsf{Env}(\bm \varepsilon), \pi}[\reg(T)]\geq 
\bbE_{\bm \varepsilon}[R_K]\bbP(N_T\geq K)\geq c_0 \log K\, \bbP(N_T\geq K). 
$$
Suppose that $K\leq T/5$ and set $p$ so that $pT\geq 5K$. By the Paley-Zygmund inequality, we have
$$
\bbP(N_T\geq K)\geq \frac{8}{25}. 
$$
Note that for the argument to be valid and such $p\in(0,1)$ to exist, we need $10\leq K\leq \big( \sqrt{(\textsf{C}_{\textsf{cov}}-1)T}\wedge T \big)/5$. Therefore, with $T\geq 2500(\textsf{C}_{\textsf{cov}}-1)^{-1}\vee 50$, we have
\begin{align*}
  \sup_{\substack{\textsf{Env}\text{ is}\\ \textsf{C}_{\textsf{cov}}-\text{smooth}}}
    \mathbb{E}_{\textsf{Env}, \pi}[\reg(T)]
& \geq \max_{\bm e\in \{\pm 1\}^{K}} \bbE_{\textsf{Env}(\bm e), \pi}[\reg(T)]\\
& \geq \bbE_{\bm \varepsilon, \textsf{Env}(\bm \varepsilon), \pi}[\reg(T)]\\
&\geq \log[(\textsf{C}_{\textsf{cov}}-1)T]\wedge \log T.
\end{align*}

Lastly, the fact that $\text{Pdim}(\mathcal{F}_{\textsf{clip}})=d$ follows because truncation does not increase
pseudo-dimension as thresholds inside $[-1,1]$ reduce to the same comparisons with
$\theta^{\T} x$, and thresholds outside this interval have a fixed outcome.
Ordinary homogeneous linear functions have pseudo-dimension $d$; the coordinate
unit vectors, with threshold zero, give the matching lower bound for $\mathcal{F}_{\textsf{clip}}$.

\end{proof-of-theorem}



\subsection{Proofs for the well-specified setting}\label{appendix:ERM}

\begin{proof-of-theorem}[\ref{thm:lower-bound-well-specified}]
Fix $d\in\bbN_+$ and let $\mathcal{X} = \{(0,0)\}\cup ([d]\times (0,1))\subset \bbR^2$. We consider the function class
$\mathcal{F}:= \{f_{\theta}\mid f_{\theta}(0,0)=0,\, f_{\theta}(j,x)=\mathbb{I}[x\geq \theta_j],\, \theta\in[0,1]^d \}$. Note that $\vc(\mathcal{F})=d$. Fix an integer $K\leq T$. We start with constructing the following dyadic tree $(Z_t=Z_t(e_1,\cdots, e_{t-1}))_{t=1}^K$, where $e\in\{0,1\}^K$. Let $Z_1 = 1/2$ and initialize $L=0$, $U=1$. For $t = 2,\cdots, K$, 
\begin{itemize}
\item if $e_{t-1}=1$, update $U = Z_{t-1}$;
\item if $e_{t-1}=0$, update $L = Z_{t-1}$;
\item set $Z_t = (L+U) / 2$. 
\end{itemize}
Clearly, for any path $e$, there exists some threshold $\tau(e)\in\bbR$ such that $e_t = \mathbb{I}[Z_t\geq \tau(e)]$ for all $t\in[K]$.  

Next, for each fixed $\bm\varepsilon\in\{0,1\}^{d\times K}$, we construct a $\nu^2$-subGaussian (thus well-specified) environment $\textsf{Env}(\bm \varepsilon)$ and a base measure $\mu_{\bm \varepsilon}$ with respect to which $\textsf{Env}(\bm \varepsilon)$ is $\Ccov$-smooth. Fix some $p\in(0,1)$ and set $N_0(j)=
0$ for $j\in[d]$. At each round $t$, sample $D_t\sim\textsf{Bernoulli}(p)$.
\begin{itemize}
\item If $D_t=0$, reveal $(x_t, y_t) = ((0,0), 0)$.
\item Otherwise, sample $J_t\sim\textsf{Cat}(d)$, set $N_t(J_t) = N_{t-1}(J_t)+1$ and $N_t(j) = N_{t-1}(j)$ for $j\neq J_t$. Define $\widetilde{N}_t(j):= \lceil N_t(j)/(1\vee \nu^2) \rceil$ for all $j\in[d]$. 
\begin{itemize}
\item If $\widetilde{N}_t(J_t)>K$, reveal $(x_t, y_t) = ((0,0), 0)$.
\item Otherwise, reveal $(x_t,y_t)$ where $x_t = (J_t, Z_{\widetilde{N}_t(J_t)}(\varepsilon_{J_t 1},\cdots, \varepsilon_{J_t (\widetilde{N}_t(J_t)-1)}))$ and $y_t = \varepsilon_{J_t \widetilde{N}_t(J_t)} + \eta_t$, where $\eta_t\sim\mathcal{N}(0,\nu^2)$ independent of everything else.
\end{itemize}
\end{itemize}

Clearly, the above environment is $\nu^2$-subGaussian with $\theta^*_j(\bm\varepsilon) = \tau(\bm\varepsilon_j)$ for $j\in[d]$. Note that for each $j\in[d]$, every node on the path preceding the last visited node, namely the $(\widetilde{N}_T(j)\wedge K)$-th node, is visited exactly $\lfloor 1\vee \nu^2\rfloor$ times.
Suppose $p$ and $K$ are selected so that $\Ccov -pK\geq 1$.  $\textsf{Env}(\bm \varepsilon)$ is $\Ccov$-smooth with respect to the base measure
$$
\mu_{\bm\varepsilon} = 
\frac{p}{d\Ccov} \sum_{1\leq j \leq d}\sum_{1\leq k\leq K}\delta_{(j,Z_{jk})}  + \Big(1 - \frac{pK}{\Ccov}\Big)\delta_{(0,0)},
$$
where $Z_{jk}:= Z_k(\varepsilon_{j1}, \cdots, \varepsilon_{j(k-1)})$. 
To see this, note that 
$$
\bbP_t(X_t = (j,Z_{jk})\mid \mathcal{H}^x_{t-1})\leq p/d\leq \Ccov \cdot\mu_{\bm\varepsilon}((j,Z_{jk})),
$$
and 
$$
\bbP_t(X_t = (0,0)\mid \mathcal{H}^x_{t-1})\leq1\leq \Ccov -pK\leq \Ccov \cdot\mu_{\bm\varepsilon}((0,0)).
$$

Now we proceed to analyze the regret. For $j\in[d]$, $k\in[\widetilde{N}_T(j)\wedge K]$ and $r\in[\lfloor 1\vee \nu^2\rfloor]$, we denote by $\tau_{jkr}$ the $r$-th visit to the $k$-th node of index $j$. Since $\textsf{Env}(\bm \varepsilon)$ is well-specified, for any algorithm $\pi$,
\begin{align*}
\bbE_{\textsf{Env}(\bm \varepsilon), \pi}\big[\reg(T)\big] &\geq \bbE_{\textsf{Env}(\bm \varepsilon), \pi}\bigg[\sum_{t=1}^T \big(\widehat{y}_t-f^\star(x_t)\big)^2\bigg]\\
&\geq
\bbE_{\textsf{Env}(\bm \varepsilon), \pi}\bigg[
\sum_{j=1}^d \sum_{k=1}^{(\widetilde{N}_T(j)\wedge K)-1}\sum_{r=1}^{\lfloor 1\vee \nu^2\rfloor} \big( \widehat{y}_{\tau_{jkr}} - \varepsilon_{jk} \big)^2
\bigg].
\end{align*}
Now suppose that $\varepsilon_{jk}\overset{\mathrm{i.i.d.}}{\sim}\textsf{Bernoulli}(1/2)$. Conditional on $(D_t, J_t)_{t=1}^T$, for $r>1$, we have 
\begin{align*}
\bbE_{\bm\varepsilon, \textsf{Env}(\bm \varepsilon), \pi}\Big[ \big( \widehat{y}_{\tau_{jkr}} - \varepsilon_{jk} \big)^2
\Big]\geq&
\bbE_{\bm\varepsilon, \textsf{Env}(\bm \varepsilon)}\bigg[ \Big( \bbE\big[\varepsilon_{jk}\,|\,y_{\tau_{jk1}},\cdots, y_{\tau_{jk(r-1)}}\big] - \varepsilon_{jk} \Big)^2
\bigg]\\
\geq&\frac{1}{8} \Big( 1-\sqrt{\frac{r-1}{2} \mathsf{KL}\big(\mathcal{N}(0, \nu^2)\,||\,\mathcal{N}(1, \nu^2)\big) } \Big)\\
\geq & \frac{1}{8} \Big( 1-\sqrt{\frac{r-1}{2\nu^2}} \Big)\geq \frac{1}{32},
\end{align*}
where the second inequality follows from a standard lower bound for the Bayes risk, which we formally stated as Lemma~\ref{lemma:lecam}. Note that the same bound holds for $r=1$. Hence, 
\begin{align*}
\bbE_{\bm\varepsilon, \textsf{Env}(\bm \varepsilon), \pi}\big[\reg(T)\big] \geq& 
\frac{1\vee \nu^2}{32} \sum_{j=1}^d \Big(\bbE\big[\widetilde{N}_T(j)\wedge K\big] - 1\Big)\\
\geq & 
\frac{1\vee \nu^2}{32} \sum_{j=1}^d \Big(K\bbP\big[N_T(j)\geq (1\vee \nu^2)K\big] - 1\Big).
\end{align*}
Note that $N_T(j)\sim\textsf{Binomial}(T, p/d)$. Suppose that $K\leq T/\big(5(1\vee \nu^2)d\big)$ and set $p$ so that $pT/d\geq 5(1\vee \nu^2)K$. Then it follows from the Paley-Zygmund inequality that 
$$
\bbP\big(N_T(j)\geq (1\vee \nu^2)K\big)\geq \frac{8}{25}.
$$
Therefore, when $K\geq 10$, we have 
\begin{align*}
\sup_{\substack{\textsf{Env}\text{ is}\\ \Ccov-\text{smooth}\\\text{and $\nu^2$-subGaussian}}}
\mathbb{E}_{\textsf{Env}, \pi}[\reg(T)]
& \geq \max_{\bm \varepsilon\in \{0,1\}^{d \times K}} \bbE_{\textsf{Env}(\bm \varepsilon), \pi}[\reg(T)]\\
& \geq \bbE_{\bm \varepsilon, \textsf{Env}(\bm \varepsilon), \pi}[\reg(T)]\\
&\geq \frac{(1\vee \nu^2)Kd}{160}.
\end{align*}
To complete the proof, note that for our argument to be valid, we can choose any $K$ satisfying
$$
10 \leq K\leq \frac{T}{5(1\vee \nu^2)d}\wedge \sqrt{\frac{(\Ccov-1)T}{5(1\vee \nu^2)d}}.
$$

\end{proof-of-theorem}

\begin{lemma}\label{lemma:lecam}
Let $\theta_0\neq \theta_1\in\bbR^d$ index two distributions on $\mathcal{X}$, $P_0$ and $P_1$, respectively. Suppose that the parameter $\theta$ follows the uniform prior $(\delta_{\theta_0}+\delta_{\theta_1})/2$ and data $X\in\mathcal{X}$ satisfies $X\,|\,\theta\sim P_{\theta}$. Then the optimal Bayes risk satisfies
\begin{align*}
\bbE\Big[ \big( \bbE[\theta\,|\,X] - \theta \big)^2 \Big]=&
\inf_{\widehat{\theta}}\bbE\Big[ \big( \widehat{\theta}(X) - \theta \big)^2 \Big]\\
\geq & \frac{\|\theta_1 - \theta_0\|_2^2}{8}
\Big( 1-\sqrt{ \mathsf{KL}(P_0||P_1) /2 } \Big). 
\end{align*}
\end{lemma}

\begin{proof-of-lemma}[\ref{lemma:lecam}]
See, for example,  \citet[Chapter 15]{wainwright_high-dimensional_2019}.
\end{proof-of-lemma}

\begin{definition}
\label{def:wills}
For $m$ vectors $z_1,\cdots,z_m\in \mathcal{X}\subseteq \mathbb{R}^d$, the Wills functional of a function class $\mathcal{F}:\mathcal{X}\to \bbR$ is defined as 
\begin{align}\label{eq:wills}
    W_m(\mathcal{F}) = \bbE_{\bm \xi}\Big[\exp\Big(\sup_{f\in \mathcal{F}} \sum_{1\leq i \leq m}\xi_i f(z_i) - \frac{1}{2}f(z_i)^2\Big)\Big],
\end{align}
where $\bbE_{\bm \xi}$ denotes expectation over $m$ independent standard Gaussian random variables.
\end{definition}

\begin{proof-of-proposition}[\ref{prop:will}]
Let $\mathcal{X}=[K]$, $\mu=\textsf{Unif}([K])$, and $\mathcal{F}:= \{f_\theta\mid f_{\theta}(k)=\theta_k,\, \theta\in\{0,1\}^K\}$. Without loss of generality, we assume $f^\star=0$. Let $N_{kt}:= \sum_{n=1}^t\mathbb{I}[X_n=k]$ be the number of visits to site $k$ up to time $t$, and define $S_{kt}:= \sum_{n=1}^t\mathbb{I}[X_n=k]\eta_t$. Clearly, $\widehat{f}_{t+1}(k)=1$ for all $\widehat{f}_{t+1}\in \widehat{\mathcal{F}}_{t+1}$ if and only if $2S_{kt}- N_{kt}> 0$. One can show that 
$$
\bbE[\reg(T)]\leq \bbE[\mathsf{Err}(T)] + \sum_{k=1}^K\bbE[(2S_{kT}-N_{kT})_+], 
$$
where
$$
\mathsf{Err}(T):= \sum_{t=1}^T \sup_{\widehat{f}_t\in\widehat{\mathcal{F}}_t}\big( \widehat{f}_t(X_t)-f^\star(X_t) \big)^2.
$$

We first derive an upper bound for $\bbE[\mathsf{Err}(T)]$. Let $\tau_{kn}:= \inf\{t: N_{kt}\geq n\}$ be the time of the $n$-th visit to site $k$. We have
$$
\mathsf{Err}(T)\leq \sum_{k=1}^K\bigg\{
1 + \sum_{n=1}^{T-1}\mathbb{I}\big[ \tau_{kn}\leq T,\, 2S_{k\tau_{kn}}-n> 0\big].
\bigg\}
$$
Recall that $Y_t=f^\star(X_t)+\eta_t$ with $\eta_t\mid \mathcal{H}_{t-1},X_t\sim \textsf{subG}(v^2)$. Fix any $k\in[K]$, and let $Z_{kt}(\lambda) = \exp\{\lambda S_{kt} - v^2\lambda^2 N_{kt}/ 2\}$. It is straightforward to verify that $\bbE [Z_{kt}(\lambda)\mid \mathcal{H}_{t-1}]\leq Z_{k(t-1)}(\lambda)$, thus $Z_{kt}(\lambda)$ is a supermartingale. By Ville's inequality, for any $a\in\bbR$, we have
\begin{align}\label{eq:villes}
\bbP \Big(\sup_{1\leq t\leq T} \Big( \lambda S_{kt}-\frac{v^2\lambda^2}{2} N_{kt} \Big)\geq a\Big)\leq e^{-a}.
\end{align}
Setting $\lambda = (2v^2)^{-1}$ and $a = n(8v^2)^{-1}$, given $\tau_{kn}\leq T$, \eqref{eq:villes} implies
\begin{align*}
\bbP \big( 2S_{k\tau_{kn}}- n > 0\big)\leq \exp \Big\{-\frac{n}{8v^2}\Big\}.
\end{align*}
Therefore, we have
$$
\bbE[\mathsf{Err}(T)]\leq K \bigg[ 1 + \sum_{n=1}^T\exp \Big\{-\frac{n}{8v^2}\Big\} \bigg]\leq \frac{K}{1-e^{-(8v^2)^{-1}}}\leq (1+8v^2)K. 
$$

To upper bound $\bbE[(2S_{kT}-N_{kT})_+]$, setting $\lambda = v^{-2}$, \eqref{eq:villes} implies
\begin{align*}
\bbP \big( 2S_{kT}- N_{kT} \geq a\big)\leq \exp \Big\{-\frac{a}{2v^2}\Big\}.
\end{align*}
Hence, 
$$
\bbE[(2S_{kT}-N_{kT})_+] = \int_{0}^{+\infty} \bbP \big( 2S_{kT}- N_{kT} \geq a\big) \d a\leq 2v^2. 
$$
Consequently, we have
$$
\bbE[\reg(T)]\leq (1+10v^2)K. 
$$

Finally, we derive the order of $\log \bbE_{\mu}[W_m(\mathcal{F})]$ when $\mu=\textsf{Unif}([K])$. Recall that given $Z_1,\cdots,Z_m\in \mathcal{X}$, 
\begin{align}
W_m(\mathcal{F}) = \bbE_{\bm \xi}\Big[\exp\Big(\sup_{f\in \mathcal{F}} \sum_{1\leq i \leq m}\xi_i f(Z_i) - \frac{1}{2}f(Z_i)^2\Big)\Big],
\end{align}
where $\xi_i\overset{\mathrm{i.i.d.}}{\sim}\mathcal{N}(0,1)$. Define $m_k:= \sum_{i=1}^m\mathbb{I}[Z_i=k]$, and suppose that $m_k\geq 1$ for all $k$. Define $G_k = (m_k)^{-1/2} \sum_{i=1}^m \mathbb{I}[Z_i=k] \xi_i$. Clearly $G_k\,|\,m_k \overset{\mathrm{i.i.d.}}{\sim}\mathcal{N}(0,1)$. Some algebra yields
$$
\sup_{f\in \mathcal{F}} \sum_{1\leq i \leq m}\xi_i f(Z_i) - \frac{1}{2}f(Z_i)^2 = \sum_{k=1}^K \Big(\sqrt{m_k}G_k - \frac{m_k}{2}\Big)_+,
$$
and 
$$
\bbE\,\exp\Big\{ \Big(\sqrt{m_k}G_k - \frac{m_k}{2}\Big)_+ \Big\} = 2\Phi(\sqrt{m_k}/2), 
$$
where $\Phi$ denotes the standard normal CDF. Hence, $W_m(\mathcal{F}) = 2^K\Pi_{k=1}^K \Phi(\sqrt{m_k}/2)\leq 2^K$. Note that the equality still holds if some $m_k$'s are zero. When $m\geq K\log(2K)$, $\bbP(\exists k\in[K]\text{ s.t. }m_k=0)\leq 1/2$. Therefore, we have $\log\bbE_{\mu}[W_m(\mathcal{F})]\asymp K$. 

\end{proof-of-proposition}

\rev{\section{Additional results}\label{sec:additional}}

\subsection{A correction to Theorem 8 of \citet{block2024performance}}\label{sec:correction}

The upper bound in \citet[Theorem~8]{block2024performance} has a noise-dependent factor $(1\vee\nu)^{1/2}$, whereas our lower bound shows that a factor of $1\vee\nu$ is necessary. This discrepancy originates in their proof of Lemma~13. Adopting their notation, let $B:= 2 (1\vee\nu) \sqrt{\log (T/\delta)}$. The inequality
$$
\mathbb{E}\left[\sup _{g \in \mathcal{G}} 8 \cdot\langle\varepsilon \cdot| \eta|, g\rangle_{T-1}-\|g\|_{T-1}^2\right] \leq B \cdot \mathbb{E}\left[\sup _{g \in \mathcal{G}} 8 \cdot\langle\varepsilon, g\rangle_{T-1}-\|g\|_{T-1}^2\right]+8 T \delta
$$
in their proof should be replaced by
$$
\mathbb{E}\left[\sup _{g \in \mathcal{G}} 8 \cdot\langle\varepsilon \cdot| \eta|, g\rangle_{T-1}-\|g\|_{T-1}^2\right] \leq B^2 \cdot \mathbb{E}\left[\sup _{g \in \mathcal{G}/B} 8 \cdot\langle\varepsilon, g\rangle_{T-1}-\|g\|_{T-1}^2\right]+ \frac{32\nu^2}{B}\Big(\frac{\delta}{T}\Big)^2.
$$
Factoring out the noise amplitude requires rescaling the quadratic penalty, yielding a quadratic rather than linear factor $B^2$. With this minor correction and the contraction property of the Wills functional, the bound in Lemma~13 becomes
$$
\mathbb{E}\left[\left\|\widehat{f}_T-f^{\star}\right\|_{T-1}^2\right] \lesssim \frac{(1\vee\nu^2)\log T}{T}   \Big(\log \mathbb{E}_{Z_{t, j}}\left[W_{k(T-1)}\left(256\left(\mathcal{F}-f^{\star}\right)\right)\right]+T e^{-\sigma k}\Big).
$$
Following the remaining steps of the proof of \citet[Theorem~8]{block2024performance} then yields an upper bound for ERM with the correct dependence on $\nu$.

\subsection{VAW estimator under a clipped linear function class}\label{sec:app_VAW}

The VAW estimator is defined as
$$
\widehat{\theta}_t \in
\operatorname*{arg\,min}_{\theta \in \bbR^d}
\Big\{
\langle \theta, x_t \rangle^2
+ \sum_{i<t} \bigl(\langle \theta, x_i \rangle - y_i\bigr)^2
\Big\}.
$$
Proposition~\ref{prop:upper-unregularized-VAW} establishes a regret upper
bound of order $O(\sqrt{d\Ccov T})$ against the unrestricted linear
function class for $d\geq2$, whereas the lower bound in
Theorem~\ref{theorem:minimax} is stated for bounded function classes. Here we provide examples showing that both the VAW estimator and its clipped version can have linear regret under the clipped linear class $\mathcal{F}_{\textsf{clip}}:= \{f_{\theta}: f_{\theta}(x)=(\theta^\T x\vee -1)\wedge 1, x\in \mathcal{X}, \theta\in\bbR^d\}$ even when $d=1$.

We start with the standard VAW estimator. Consider the deterministic sequence
$$
x_{2k-1}=1,\qquad x_{2k}=2,\qquad y_t=1.
$$
Taking $\mu=\delta_1/2+\delta_2/2$, this environment satisfies the smoothness condition with $\Ccov=2$, even conditional on the full history. In this case,
$$
f_1\in\operatorname*{arg\,min}_{f\in\mathcal F_{\textsf{clip}}}
\sum_{t=1}^T(f(x_t)-y_t)^2,\quad \text{and}\quad 
\sum_{t=1}^T(f_1(x_t)-y_t)^2=0.
$$
Let $\widehat y_t=\widehat\theta_t x_t$ denote the VAW prediction. At each odd round $t=2k+1$, the preceding observations contain $k$ copies of each input. Hence
$\widehat\theta_{2k+1}=3k/(5k+1)$. 
Since $x_{2k+1}=1$ and $\widehat{y}_{2k+1}<3/5$, 
the VAW estimator incurs constant squared loss on every odd round. Consequently, 
$$
\sum_{t=1}^T(\widehat y_t-y_t)^2- 
\inf_{f\in\mathcal F_{\textsf{clip}}}
\sum_{t=1}^T(f(x_t)-y_t)^2
=\Omega(T).
$$
For the same construction, it is clear that directly clipping the prediction $\widehat y_t$ still yields linear regret.

Define $c(x)=(x\vee -1)\wedge 1\ (x\in\bbR)$ as the clipping function. We next consider the following clipped VAW estimator, 
$$
\widehat{\theta}_t \in
\operatorname*{arg\,min}_{\theta \in \bbR^d}
\Big\{
c(\langle \theta, x_t \rangle)^2
+ \sum_{i<t} \bigl(c(\langle \theta, x_i \rangle) - y_i\bigr)^2
\Big\}.
$$
We first take $(x_1,y_1)=(x_2,y_2)=(1,1)$, and subsequently alternate between
$$
(x_{2k+1},y_{2k+1})=(1,1),\quad
(x_{2k+2},y_{2k+2})=(4/3,1/2),
\qquad k\geq 1.
$$
Taking $\mu=\delta_1/2+\delta_{4/3}/2$, this environment satisfies the smoothness condition with $\Ccov=2$, even conditional on the full history. Let $\widehat y_t=c(\widehat\theta_t x_t)$ denote the clipped prediction. At round $2k+1$, the preceding observations contain $k+1$ copies of $(1,1)$ and $k-1$ copies of $(4/3,1/2)$. Minimizing the objective yields
$\widehat{\theta}_{2k+1} = (15k+3) / (25k+2)$ and $\widehat{\theta}_{2k+2} \geq 1$, thus $\widehat{y}_{2k+1}=(15k+3)/(25k+2)\leq 2/3$ and 
$\widehat{y}_{2k+2}=1$. The comparator $f_1$ predicts $1$ at both inputs. Thus the estimator incurs the same loss as $f_1$ on every round with input $4/3$, but incurs excess loss at least $1/9$ on every odd round after the first two rounds. Therefore, in this example, the clipped VAW estimator also incurs linear regret.

\subsection{Smoothness conditions for empirical risk minimization}\label{sec:app_ERM}

Our working assumption is the partial-history smoothness condition in Definition~\ref{def:smoothness}, which requires
\begin{align*}
\frac{\mathrm{d}P_t(\,\cdot\mid\mathcal{H}_{t-1}^{x})}{\mathrm{d}\mu}(x)
\leq \Ccov
\qquad \textit{for $\mu$-a.e. $x\in\mathbb{R}^d$,}
\end{align*}
where $\mathcal{H}_{t-1}^{x}=\sigma(\{x_1,\dots,x_{t-1}\})$ denotes the partial covariate history, $\mu$ denotes the base measure, and $\Ccov\geq 1$. Full-history smoothness is a stronger condition replacing $\mathcal{H}_{t-1}^{x}$ with $\mathcal{H}_{t-1} = \sigma(\{x_s,\widehat y_s,y_s\}_{1\leq s\leq t-1})$, the full interaction history. Partial-history smoothness does not control the additional concentration caused by conditioning on past responses. For example, a past response may reveal the next covariate even when the covariates themselves are independent and identically distributed.

In the well-specified setting, although the lower bound in Theorem~\ref{thm:lower-bound-well-specified} is stated under partial-history smoothness, the hard environment in our construction in fact satisfies full-history smoothness. To see this, after fixing the path indexing each environment and its reference measure, the activation and coordinate choices at each round are sampled independently of the past. The verification of smoothness therefore remains valid when conditioning on $\mathcal H_{t-1}$ instead of $\mathcal H_{t-1}^x$. Consequently, the lower bound holds under both smoothness conditions.

In contrast, the ERM upper bound on cumulative estimation error in Theorem~8 of \citet{block2024performance} requires full-history smoothness. This assumption is used explicitly in their decoupling inequality, Lemma~12, and its proof in Appendix~C.1. To show that this distinction is not merely a limitation of the proof technique, we next construct a well-specified environment satisfying partial-history smoothness in which the ERM is unique and incurs linear expected cumulative estimation error, and hence linear expected regret.

Fix any horizon $T\geq 2$, and let $c(u)=u-\lfloor u\rfloor$ 
be the fractional part of $u$. Take $y_1,\ldots,y_T$ i.i.d. from $\operatorname{Unif}[-1/4,1/4]$, and independently draw $x_1\sim\operatorname{Unif}[0,1]$. Define
$$
\overline y_t=\frac1t\sum_{s=1}^t y_s,
\qquad
x_{t+1}=c(T\overline y_t),
\qquad 1\leq t<T.
$$
Consider the function class
$\mathcal{F}_T=\{f_\theta:\theta\in[-1/4,1/4]\}$, where
$$
f_{\theta}(x) = \begin{cases}
0&\quad \text{for }\theta=0, \\
\theta+(1-\theta)\mathbf1_{{x=c(T\theta)}}&\quad \text{for }\theta\neq 0.
\end{cases}
$$
Note that $\mathcal F_T$ has pseudo-dimension $2$.

Taking $f^\star=f_0$, this environment is well-specified with bounded, conditionally mean-zero noise. It also satisfies partial-history smoothness with respect to $\mu=\operatorname{Unif}[0,1]$, with a horizon-independent smoothness constant $\Ccov$. To see this, conditional on $\sigma(\{x_1,y_1,\ldots,y_{t-2}\})\supseteq \mathcal H_{t-1}^x $, $x_t$ is the fractional part of a uniform random variable on an interval of length $T/(2(t-1))\geq1/2$. Its density is therefore bounded by a universal constant, and the same bound holds conditional on the smaller history $\mathcal H_{t-1}^x$. In contrast, $x_t$ is determined by the past responses for every $t\geq2$, so full-history smoothness fails.

We next derive the ERM and show that it is almost surely unique at every round $t\geq2$. For any $\theta$, since $y_s\in[-1/4,1/4]$,
$$
\sum_{s<t}(f_\theta(x_s)-y_s)^2
\geq
\sum_{s<t}(\theta-y_s)^2
\geq
\sum_{s<t}(\overline{y}_{t-1}-y_s)^2.
$$
Since the covariates are pairwise distinct and $\overline y_{t-1}\neq0$ almost surely, the function $f_{\overline{y}_{t-1}}$ has no spike on the preceding observations and uniquely attains both inequalities. Consequently, the ERM $\widehat{f}_t=f_{\overline{y}_{t-1}}$ is unique, and by construction it satisfies 
$\widehat{f}_t(x_t)=1$ for $t\geq 2$.

Therefore, the expected cumulative estimation error is at least $T-1$. Since it is no greater than the expected regret in the well-specified setting, ERM incurs $\Omega(T)$ expected regret.

\subsection{From error to regret under well-specification}\label{sec:app_error_regret}

In the well-specified setting under full-history smoothness, \citet{block2024performance} establishes an upper bound for the expected cumulative estimation error of ERM. Here we show that the same rate holds for the expected regret against the empirical best comparator.\\
Conditional mean-zero noise gives
$$
\bbE[\reg(T)] =
\mathbb E\left[\sum_{t=1}^T
(\widehat f_t(x_t)-f^\star(x_t))^2\right]
+\Delta_T,
$$
where the additional term is
\begin{align*}
\Delta_T&= \mathbb E\left[
\sup_{f\in\mathcal F}\sum_{t=1}^T
\left(2\eta_t \big[f(x_t) - f^{\star}(x_t)\big]-\big[f(x_t) - f^{\star}(x_t)\big]^2\right)
\right]\\
&=T\cdot\mathbb{E}\left[\sup _{g \in \mathcal{G}} 2 \cdot\langle \eta, g\rangle_{T}-\|g\|_{T}^2\right],
\end{align*}
where $\mathcal{G} = \mathcal{F}-f^{\star}$. Since $|g|\leq 2$, for any $a\geq 1$, we have
\begin{align*}
\mathbb{E}\left[\sup _{g \in \mathcal{G}} 2 \cdot\langle \eta, g\rangle_{T}-\|g\|_{T}^2\right]\leq 
a\cdot\mathbb{E}\left[\sup _{g \in \mathcal{G}} 2 \cdot\langle \eta/a, g\rangle_{T}-\|g\|_{T}^2\right] + 4(a-1).
\end{align*}
The proof of \citet[Lemma~13]{block2024performance}, together with the correction in Appendix~\ref{sec:correction}, establishes the following upper bound
\begin{align*}
\mathbb{E}\left[\sup _{g \in \mathcal{G}} 2 \cdot\langle \eta/a, g\rangle_{T}-\|g\|_{T}^2\right]
\lesssim \Big(1\vee\frac{\nu^2}{a^2}\Big)T^{-1} Q_T.
\end{align*}
where $Q_T = (\log T)  (1\vee \log \mathbb{E}_{\mu}\left[W_{2\Ccov T\log T}\left(256\mathcal{F}\right)\right])$. Hence, we have
$$
\Delta_T \lesssim \Big(a\vee\frac{\nu^2}{a}\Big) Q_T + (a-1)T. 
$$
When $Q_T\leq T$, setting $a = 1\vee \sqrt{(1\vee \nu^2)T^{-1}Q_T}$ yields 
$$
\Delta_T\lesssim \sqrt{(1\vee \nu^2)T Q_T}.
$$
When $Q_T >T$, the same bound follows from 
$$
\Delta_T\leq4\sum_{t=1}^T\bbE |\eta_t|\lesssim \nu T.
$$
Since $\Ccov\geq1$, the additional term $\Delta_T$ is absorbed into the rate in equation~\eqref{eq:ERM_upper_bound}. Thus, under the same full-history smoothness assumption, that rate for the expected cumulative estimation error of ERM also bounds expected regret against the empirical best comparator.